\documentclass[11pt, a4paper]{article}
 
\usepackage[
  top=1.1in, bottom=1.1in,
  left=0.9in, right=0.9in
]{geometry}
 
\usepackage[T1]{fontenc}
\usepackage[utf8]{inputenc}
\usepackage{microtype}
 
\usepackage{amsmath, amssymb, amsthm}
 
\usepackage{xcolor}
\definecolor{accentcolor}{RGB}{0, 74, 135}   % deep academic blue — change to taste
 
\usepackage{graphicx}
\usepackage{booktabs}
\usepackage{caption}
\usepackage{subcaption}
\usepackage{float}
 
\usepackage[
  colorlinks = true,
  linkcolor  = accentcolor,
  citecolor  = accentcolor,
  urlcolor   = accentcolor
]{hyperref}
\usepackage{url}
 
\usepackage[
  backend      = biber,
  style        = numeric,
  sorting      = none
]{biblatex}
\usepackage{fancyhdr}
\usepackage{setspace}
\usepackage{titlesec}
\titleformat{\section}
  {\normalsize\scshape\centering\color{accentcolor}}{\Roman{section}.}{0.6em}{}
\titleformat{\subsection}
  {\normalsize\scshape\color{accentcolor}}{\thesubsection.}{0.5em}{}
\titleformat{\subsubsection}
  {\small\itshape}{\thesubsubsection.}{0.5em}{}
 
\usepackage{abstract}

\theoremstyle{definition}  % 'definition' style uses upright font
\newtheorem{theorem}{Theorem}[section]

\newtheorem{corollary}[theorem]{Corollary}
\theoremstyle{remark}

\newcommand{\calX}{\mathcal{X}}
\newcommand{\calS}{\mathcal{S}}

\newcommand{\calW}{\mathcal{W}}
\newcommand{\calP}{\mathcal{P}}
\newcommand{\eps}{\epsilon}
 
\usepackage{lipsum}
\usepackage{enumitem}
\setlist{topsep=4pt}
\setlist[itemize]{itemsep=4pt}
 
\usepackage{xspace}

\usepackage[capitalize,noabbrev]{cleveref}
\usepackage{algorithm}
\usepackage{algpseudocode}
\usepackage{comment}
 
\usepackage{titletoc}
\usepackage{tikz}
\theoremstyle{definition}
\newtheorem{claim}[theorem]{Claim}
\usepackage[textsize=tiny]{todonotes}
\newcommand{\arcmark}{\hbox{\textsc{ArcMark}}\xspace}
 
\title{%
  \Large\scshape\textcolor{accentcolor}{ArcMark}%
  \textsc{: Distortion-Free Multi-Byte LLM Watermark}\\
  \textsc{via Optimal Transport}\\[0.3em]
}
 
\author{%
  \begin{tabular}{ccc}
    Atefeh Gilani\textsuperscript{\textcolor{accentcolor}{1}} &
    Sajani Vithana\textsuperscript{\textcolor{accentcolor}{2}} &
    Carol Xuan Long\textsuperscript{\textcolor{accentcolor}{2}} \\[0.3em]
    Oliver Kosut\textsuperscript{\textcolor{accentcolor}{1,$\dagger$}} &
    Lalitha Sankar\textsuperscript{\textcolor{accentcolor}{1,$\dagger$}} &
    Flavio P.\ Calmon\textsuperscript{\textcolor{accentcolor}{2,$\dagger$}}
  \end{tabular}
}
 
\date{}
 
\begin{document}
 
\maketitle
\thispagestyle{empty}
% Force affiliations into footnote area without any marker
\makeatletter
\insert\footins{%
  \reset@font\footnotesize
  \interlinepenalty\interfootnotelinepenalty
  \splittopskip\footnotesep
  \splitmaxdepth\dp\strutbox
  \floatingpenalty\@MM
  \hsize\columnwidth
  \@parboxrestore
  \protected@edef\@currentlabel{}%
  \color@begingroup
  \noindent
  \textsuperscript{\textcolor{accentcolor}{1}}Arizona State University\\
  \textsuperscript{\textcolor{accentcolor}{2}}Harvard University\\
  \textsuperscript{\textcolor{accentcolor}{$\dagger$}}Equal senior author contribution\\
  Correspondence may be sent to Atefeh Gilani (\href{mailto:atefehhgilanii@gmail.com}{atefehhgilanii@gmail.com})
  \color@endgroup}
\makeatother

% ── Abstract ─────────────────────────────────────────────────
\begin{abstract}
Watermarking is an important tool for promoting the responsible use of large language models (LLMs). Existing watermarks insert a signal into generated tokens that either flags LLM-generated text (zero-bit watermarking) or encodes more complex messages (multi-bit watermarking). Though a number of recent approaches insert multiple bits into text without perturbing average next-token predictions, they largely extend design principles from the zero-bit setting, such as encoding a single bit per token. In contrast, a watermarker capable of embedding multiple \emph{bytes} into the text would dramatically increase the potential applications, by embedding information such as the ID of the user who submitted the prompt, the precise model version that was used, or even the prompt itself. We address this problem by introducing \arcmark: a new watermark construction based on coding and information-theoretic principles that is capable of reliably embedding multiple bytes of information into just a few hundred tokens, without any distortion of the underlying LLM next-token distribution. We derive \arcmark by formulating the distortion-free watermarking problem as a channel coding problem, and deriving an information-theoretic channel capacity that establishes the fundamental limit of embedding information in LLM output in a distortion-free manner. This capacity formulation informs the design of \arcmark. In practice, \arcmark outperforms competing multi-bit distortion-free watermarks in terms of reconstruction accuracy, including in the face of attacks that alter a subset of the LLM text. \arcmark output is also shown to be indistinguishable from unwatermarked text in terms of perplexity, and in downstream task quality.
\end{abstract}

% ============================================================
\section{Introduction}

The process of embedding several bits of information into tokens generated by a large language model (LLM) is commonly referred to as \emph{multi-bit watermarking}. Multi-bit watermarks can encode, for example, which model and user generated a given piece of text or code. Multi-bit watermarks can also help AI providers respond to emerging policy and regulatory efforts that call for marking their outputs as AI-generated and curbing LLM misuse \cite{ncsl_ai_legislation_2024,rijsbosch2025adoption,chandra2024reducing}.

Ideally, a multi-bit watermark should maximize \emph{rate}---the number of bits encoded per token---and minimize the \emph{error probability} of decoding the watermarked message. Recent constructions balance these two objectives while preserving text quality by enforcing a \emph{distortion-free} constraint: averaged over side information shared between a watermark encoder and decoder, watermarking does not change the LLM's average next-token predictions. Side information is usually generated by hashing previously-generated tokens and shared secret keys \cite{kirchenbauer2023watermark,dathathri2024scalable}. 

Multi-bit watermarking is a more complex counterpart of \emph{zero-bit watermarking}, which aims only to decide whether a text is LLM-generated or not \cite{kirchenbauer2023watermark,tsurheavywater,aaronson2023watermark,kuditipudi2023robust,dathathri2024scalable}. Despite being a fundamentally different problem (statistical detection vs.\ communication), existing multi-bit methods often extend design principles from the zero-bit setting. For example, constructions such as \cite{feng2025bimark} and \cite{mpac} encode information on a token-by-token basis, rather than treating message recovery over long sequences of tokens as a channel coding problem. As a result, current watermarking methods reliably embed only a few bits into sequences of hundreds of tokens. 

For multi-bit watermarks to be useful in practice, they must allow reliable recovery of several \emph{bytes} of information over strings of tokens. This raises two questions:
\begin{enumerate}[label=\textbf{Q\arabic*)}, align=right, leftmargin=*, labelsep=.2em]
    \item In theory, what is the largest amount of information we can reliably embed into LLM-generated text without distorting next-token predictions? 
    \item Can this limit be approached by a practical multi-\emph{\underline{byte}} watermarking scheme?
\end{enumerate}
We provide answers to these questions by deriving fundamental information-theoretic limits for watermarking and introducing  \arcmark: a watermarking method built on coding theory that achieves reliable multi-byte insertion in text (see Fig. \ref{fig:llama3_qwen-token300}). 
\begin{figure}[!tb]
    \centering
    \begin{minipage}{0.32\linewidth}
        \centering
        \includegraphics[width=\linewidth]{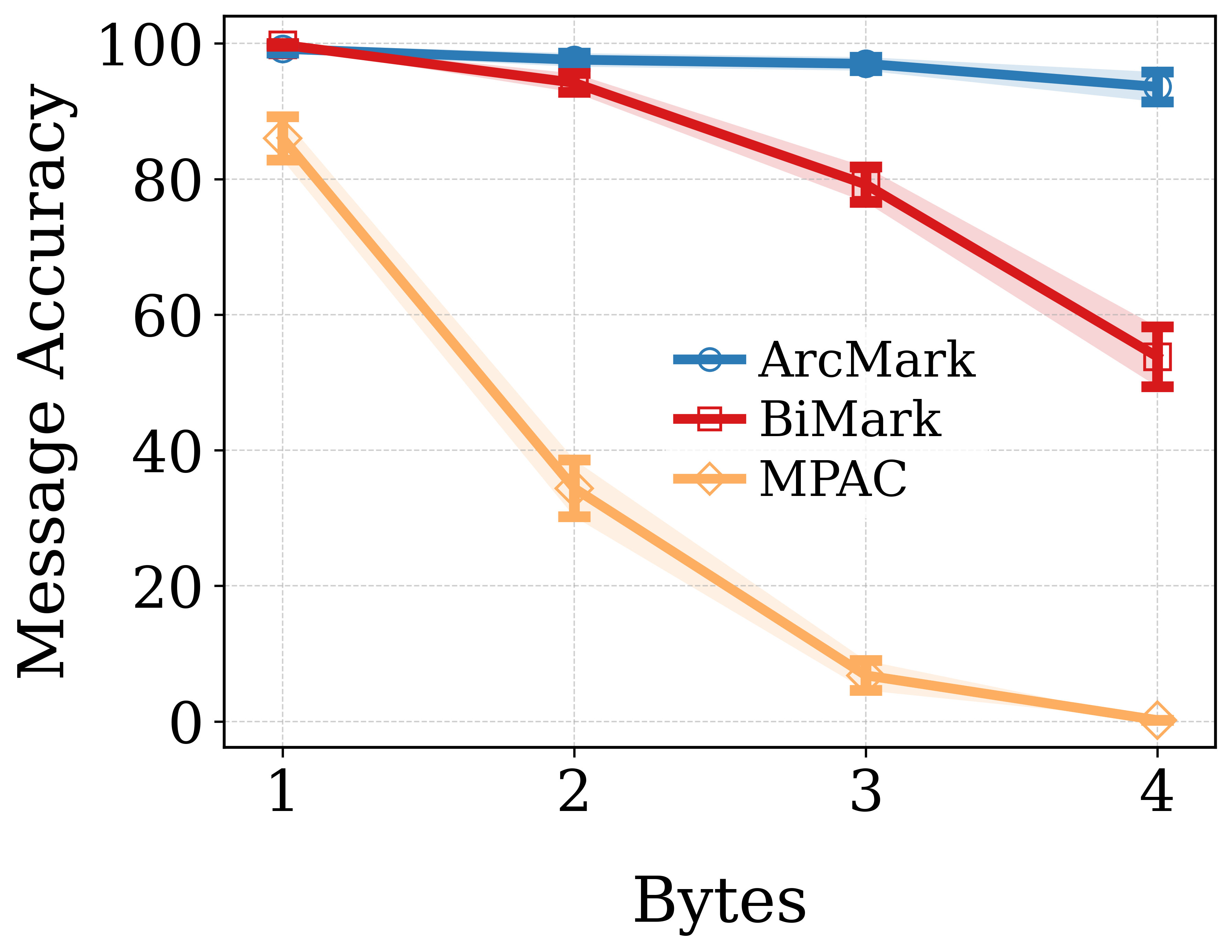}
    \end{minipage}
    \hfill
    \begin{minipage}{0.32\linewidth}
        \centering
        \includegraphics[width=\linewidth]{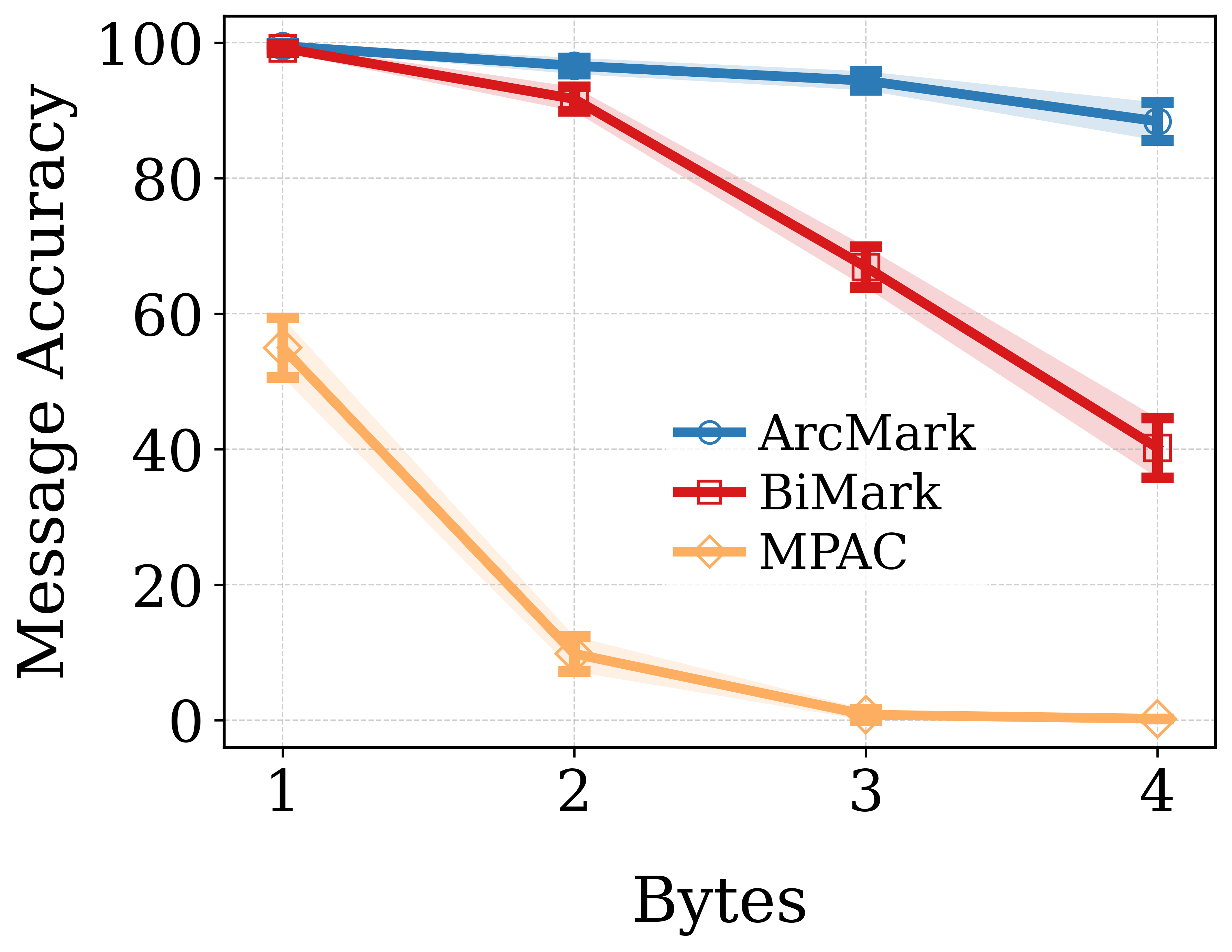}
    \end{minipage}
    \hfill
    \begin{minipage}{0.32\linewidth}
        \centering
        \includegraphics[width=\linewidth]{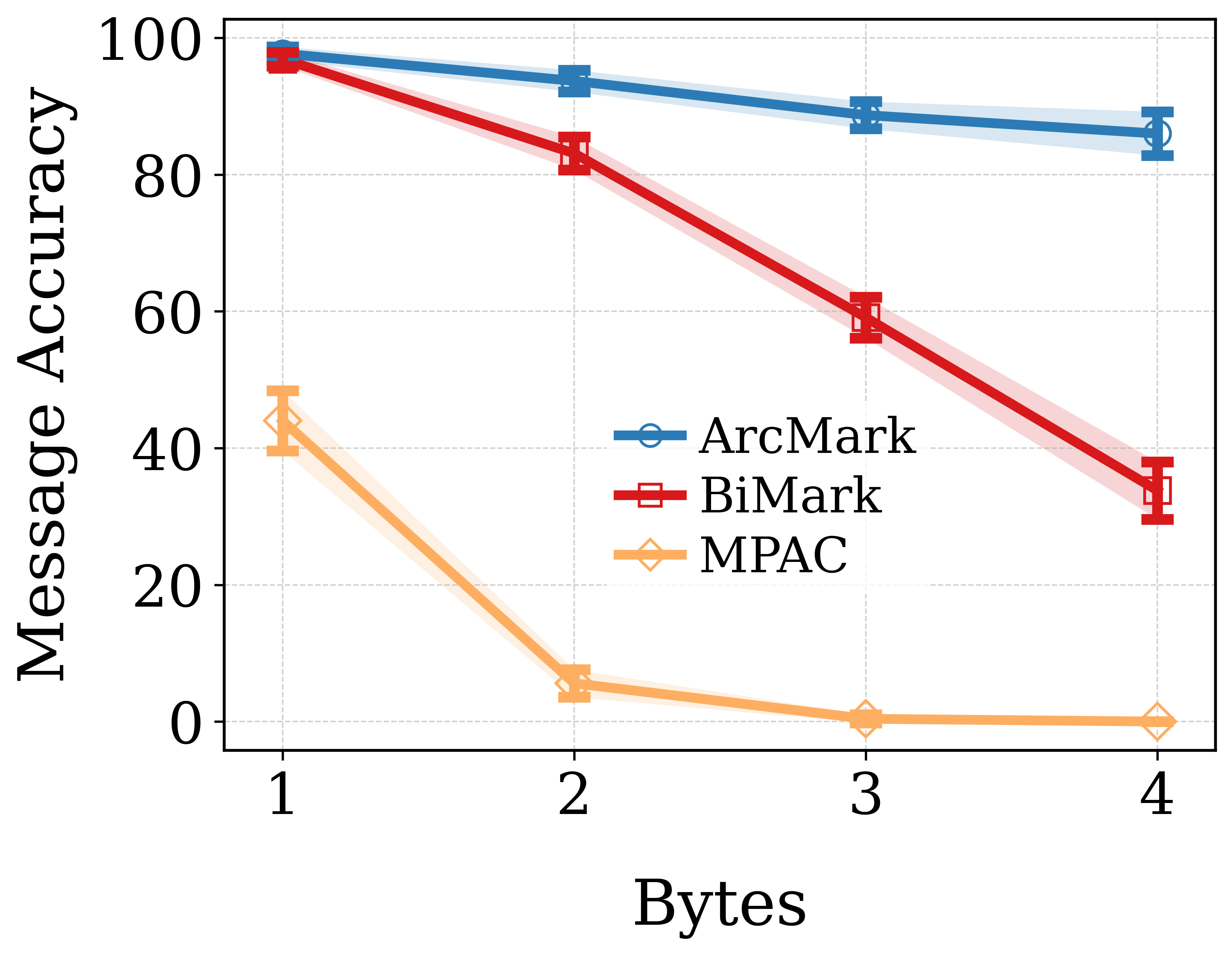}
    \end{minipage}
  \caption{Message accuracy at 300 tokens as a function of watermark payload size on Llama3-8B (left), Mistral-7B (middle), and Qwen3-8B (right), averaged over 1000 trials for payloads of 1, 2, and 3 bytes and 500 trials for 4 bytes. Error bars indicate standard error of the mean (SEM).}
    \label{fig:llama3_qwen-token300}
\end{figure}
We formally define and derive the (Shannon) capacity of multi-bit LLM watermarking, i.e., the largest number of watermarked bits that a token can carry. 

Our key insight is modeling multi-bit watermarking as communication over a noisy channel with side information available at both the encoder (the watermarker) and decoder. 
Under this model, the Shannon capacity of distortion-free multi-bit watermarking is defined as the maximum achievable rate (bits/token) while ensuring asymptotically vanishing decoding error probability over long sequences of tokens. 

Our information-theoretic analysis inspires \arcmark. \arcmark uses a random linear channel code to encode information about the message into each token, rather than assigning individual message bits to individual tokens. 
This linear code is defined on the integers with modulo operations, which can be naturally mapped to a circle, the ``arc'' in \arcmark. Specifically, \arcmark (i) 
represents message codeword symbols, tokens, and side information as points on the unit circle, and (ii) solves an optimal transport problem to map from codeword symbols to tokens while assuring the correct token distribution from the LLM.  

We evaluate \arcmark along five axes: message accuracy, perplexity, robustness to attacks, downstream utility, and zero-bit detection. Notably, we focus on \textbf{message accuracy}, the probability of correctly recovering the entire message, rather than \textbf{bit accuracy}, the average number of message bits that are recovered correctly. Message accuracy is more operationally relevant, as even a single incorrect bit could change the interpretation of the entire watermark. Across multiple LLMs, \arcmark attains significantly higher message accuracy than prior approaches, especially for longer message size (see Fig. \ref{fig:llama3_qwen-token300}), exhibits stronger robustness to attacks, and achieves downstream utility comparable to the unwatermarked baseline.

In summary, our main contributions include:
\begin{itemize}[leftmargin=*]
    \item We introduce \arcmark, a new distortion-free multi-bit watermark based on random linear channel coding and optimal transport.
    \item We formulate an information-theoretic model for multi-bit LLM watermarking, and derive an expression for the Shannon capacity of distortion-free watermarking.
    \item We prove that, under simplifying assumptions, \arcmark achieves the Shannon capacity.
    \item Our experiments demonstrate that \arcmark Pareto-dominates competing multi-bit watermarks in terms of rate (i.e., number of bits per token) and decoding error probability. The perplexity and downstream utility performance are nearly identical to the non-watermarked baseline, confirming the distortion-free characteristic. We also find that it is strongly robust against text substitution attacks, and achieves good zero-bit detection rates.
\end{itemize}

Our results reframe multi-bit watermarking as a problem in channel coding with side information---a formulation that dates back to Claude Shannon himself \cite{shannon1958channels}. This shift in perspective replaces heuristics with questions of capacity, reliability, and code design. Our findings also suggest that information and coding theory provide both a mathematical foundation for understanding the limits of LLM watermarks and a constructive path toward designing watermarks that embed bytes (not just bits!) into AI-generated text.

% ============================================================
\subsection{Related Work}
We review the zero-bit and multi-bit watermarking methods most closely related to \arcmark; a more comprehensive discussion is deferred to Appendix~\ref{app:related-work}.

\textbf{Zero-bit Watermarks.} Zero-bit watermarking schemes seek to detect whether a given piece of text is AI-generated. The task is formulated as a binary hypothesis test and has been studied through information-theoretic, statistical, and cryptographic lenses~\cite{kirchenbauer2023watermark,dathathri2024scalable,tsurheavywater,chen2000design,moulin2003information,martinian2005authentication,christ2024pseudorandomerrorcorrectingcodes,pmlr-v247-christ24a}. The first watermark for LLMs was proposed by \cite{kirchenbauer2023watermark}, commonly referred to as the Red-Green watermark. 
Among these, the closest to our work is \cite{tsurheavywater}, which employs an optimal-transport construction. We extend this by combining optimal transport with channel coding to embed multi-bit messages.

\textbf{Multi-bit Watermarks.}
Non-distortion-free multi-bit watermarks include \cite{qupaper}, \cite{xu2026xmarkreliablemultibitwatermarking}, and MPAC~\cite{mpac}. Distortion-free methods include \cite{cui2026mc2markdistortionfreemultibitwatermarking}, \cite{jiang2026mirrormarkdistortionfreemultibitwatermark}, \cite{stealthink}, and BiMark~\cite{feng2025bimark}, which embeds one bit per token. \cite{cui2026mc2markdistortionfreemultibitwatermarking}, \cite{jiang2026mirrormarkdistortionfreemultibitwatermark}, and \cite{stealthink} do not provide public implementations and we therefore cannot benchmark against them.

Achieving distortion-free watermarking can be considered the gold standard of the LLM watermarking problem, as it embeds messages without compromising text quality, making our problem setting strictly more challenging than that of non-distortion-free methods. We therefore benchmark  \arcmark primarily against BiMark~\cite{feng2025bimark}, and include MPAC~\cite{mpac} only as a reference point.

\section{Problem Statement}\label{problem}
\paragraph{Notation.} Random variables are denoted by uppercase letters (e.g., $X$ and $S$), with their realizations represented by lowercase letters (e.g., $x$ and $s$), and their distributions by subscripted symbols $P$ or $Q$ (e.g., $P_X$ and $P_S$). For a positive integer $m$, we define $[1\!:\!m] \triangleq \{1, \ldots, m\}$. The set $\{0,1\}^k$ denotes all binary strings of length $k$. The set of positive integers is denoted by $\mathbb{Z}^+$. For integers $a$ and $p$, $a \bmod p$ denotes the remainder of $a$ divided by $p$. We write $V \sim \text{uniform}(\mathcal{R})$ to denote that $V$ is uniformly distributed over $\mathcal{R}$. We use $\mathbf{1}_{A}$ for the indicator function on set $A$. We use the following notation for information-theoretic quantities: $H(X)$, $H(X|Y)$, $I(X;Y)$ and $I(X;Y|Z)$ for entropy, conditional entropy, mutual information, and conditional mutual information respectively.

We consider a large language model (LLM) with token vocabulary $\mathcal{X} = [1\!:\!N]$ that generates text autoregressively. At time $t$, a token $X_t\in\mathcal{X}$ is drawn according to the conditional distribution $P_{X_t|X_{1:t-1}}$, where $X_{1:t-1}$ denotes the tokens generated at previous time instances. For notational simplicity, we write $Q_{X_t}=P_{X_t|X_{1:t-1}}\in\Delta_{\mathcal{X}}$ where $\Delta_{\mathcal{X}}$ denotes the probability simplex over $\mathcal{X}$. 

%As depicted in Figure~\ref{fig:multi-water}, 
In multi-bit watermarking for LLM-generated text, there exists two parties: a watermarker (the LLM) and a decoder. At each generation step $t$, the two parties share side information $S_t\in\mathcal{S}$. The watermarker has access to the model’s next-token distribution $Q_{X_t}$, whereas the decoder observes only the generated tokens and the corresponding side information, i.e., $(X_t,S_t)$. The watermarker embeds a $k$-bit message $M\in\mathcal{M}=\{0,1\}^k$ into a sequence of $n$ generated tokens $\{X_t\}_{t=1}^n$ by sampling each $X_t$ from a watermarked distribution $Q_{X_t|S_t,M}$ that depends on the message to be embedded, the shared random key, and the original LLM token distribution $Q_{X_t}$. The watermark embedding process must preserve the quality of the generated text. We formalize this requirement via a distortion-free constraint, requiring the marginal token distribution to remain unchanged:
 \begin{align}\label{distortion-free}
     \mathbb{E}_{S_t}[Q_{X_t|S_t,M=m}]=Q_{X_t},\quad\forall m\in\mathcal{M}, \  t\in[1:n]
 \end{align}
This condition ensures that watermarking preserves the LLM's output distribution in expectation over the side information, regardless of which message is embedded.

 The decoder reconstructs the message from the observed sequence via a decoding function $g:\mathcal{X}^n\times\mathcal{S}^n\mapsto\mathcal{M}$ as $\hat{M}=g(X_{1:n},S_{1:n})$. We characterize the system by the following quantities: the error probability $P_{e}=\mathrm{Pr}(M\neq\hat{M})$, number of message bits $k$, and the token length $n$.

\section{\arcmark: Distortion-Free Multi-Byte Watermark}\label{sec:arcmark}

\begin{figure*}
    \centering
    
    \includegraphics[width=1\textwidth]{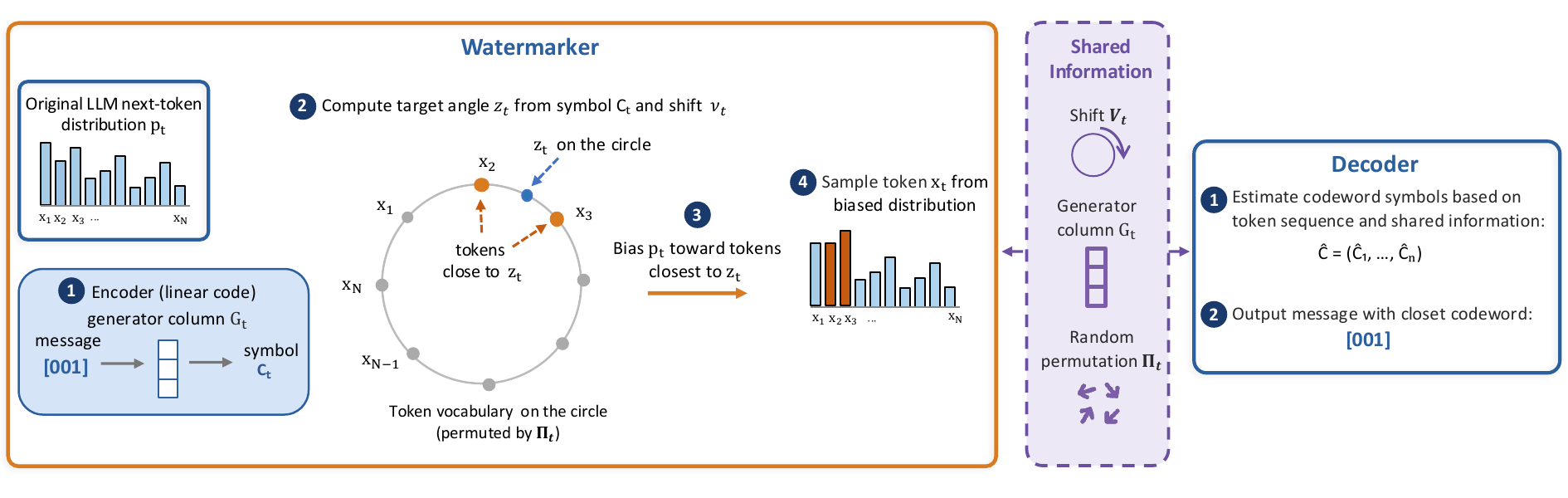}
    \caption{Overview of \arcmark}

    \label{method}
\end{figure*}
In this section, we present \arcmark, a distortion-free multi-bit watermark that embeds several bytes of information within a sequence of LLM-generated tokens while ensuring reliable decoding. In Section \ref{cap} we derive the capacity of the LLM watermarking channel, and prove that---under certain simplifying assumptions---\arcmark is capacity achieving. \arcmark consists of three stages:
\begin{enumerate}[leftmargin=*]
    \item \textbf{Message encoding:} At each time step $t$, the message to be embedded is mapped to a encoded symbol $z_t$, a point on the unit circle, via a random linear code. 
    \item \textbf{Message embedding:} At each time step $t$, the watermarker embeds the corresponding codeword symbol by sampling token $X_t$ from a watermarked distribution $Q_{X_t|Z_t}$, where $Z_t$ is derived from both the codeword symbol and the shared side information $S_t$.
    \item \textbf{Message decoding:} The detector observes the generated token-key sequence $\{(X_t, S_t)\}_{t=1}^n$ and decodes the message $\hat{M}$.
\end{enumerate}

We detail each of these stages as follows. The method is also illustrated in Fig.~\ref{method} (see also the algorithms in Appendix~\ref{app:alg}).

\paragraph{Message encoding:} 
The shared side information at time $t$ consists of three components: $S_t=(G_t, V_t,\Pi_t)$, where $G_t\sim\text{uniform}(\mathcal{F}^k)$ with $\mathcal{F}=[0:p-1]$ for some $p\in\mathbb{Z}^+$, $V_t\sim\text{uniform}(\mathcal{R})$ with $\mathcal{R}=[0:r-1]$ for some $r\in\mathbb{Z}^+$, and $\Pi_t$  is a random permutation of token indices in the LLM vocabulary. The vector $G_t$ represents a column of a generator matrix for a linear code used to map the message to a codeword symbol, $V_t$ is a random shift around the circle, and $\Pi_t$ is a permutation applied to the set of tokens. At time $t$, the watermarker maps the $k$-bit message $m\in\{0,1\}^k$ to a codeword symbol in $\mathcal{F}$ via the inner product
\begin{align}\label{linear_code}
    C_m(t)=m\cdot G_t\text{ mod }p.
\end{align}

We represent tokens $\mathcal{X}$, message codeword symbols $\mathcal{F}$, and random keys $\mathcal{R}$ as angles on the unit circle. Specifically, the $i$th token is mapped to angle $\frac{i}{N}2\pi$, the $i$th message symbol to $\frac{i}{p}2\pi$, and the $i$th random key to $\frac{i}{r}2\pi$.
\footnote{In our experiments, the side information is generated by hashing a window of previously generated tokens together with a shared secret key drawn from a pseudorandom generator. Full details are given in Sec.~\ref{sec:exp}.}

To embed message $m$ at time $t$, the watermarker combines the codeword symbol $C_m(t)\in\mathcal{F}$ with the shared key $V_t=v_t\in\mathcal{R}$ to produce the \emph{channel input}:
\begin{align}
    z_t=\left(\frac{2\pi C_m(t)}{p}+\frac{2\pi v_t}{r}+\phi\right)\bmod 2\pi \label{eq:zt-map}
\end{align}
where $\phi$ is a fixed angle offset. Next, we use optimal transport to select a token $X_t$ close to $z_t$.

\paragraph{Message embedding:}
At time $t$, the encoder transmits $z_t$ by biasing the LLM's token distribution such that, after observing $n$ token-key pairs $\{(x_t,s_t)\}_{t=1}^n$, the decoder can reliably recover the embedded message $m$. Our key insight is to design a watermarked distribution $Q_{X_t|Z_t}$ (channel) that is biased toward tokens that are closer to $z_t$ in angular distance, while maintaining the distortion-free constraint in \eqref{distortion-free}. In other words, the watermarker's objective is to sample a token $x_t$ that is as close as possible to $z_t$ in angular distance while satisfying \eqref{distortion-free}.

We define the angular distance between angles $\theta_1,\theta_2$ as:
\begin{align}\label{lee}
    d(\theta_1,\theta_2)=\min\{|\theta_1-\theta_2|,\ 2\pi-|\theta_1-\theta_2|\}
\end{align}
The watermarked distribution $Q^*_{X_t|Z_t}$ is obtained by solving the following optimization problem at each time $t$:
\begin{align}\label{ot}
    Q_{X_t|Z_t}^*&=\arg\min_{Q_{X_t|Z_t}}\mathbb{E}_{X_t,Z_t}\left[d\left(\frac{2\pi\cdot\Pi_t(X_t)}{N},Z_t\right)\right]
    \nonumber\\
    &\qquad\text{s.t. }\mathbb{E}_{Z_t}[Q_{X_t|Z_t}]=Q_{X_t}
\end{align}
where $\Pi_t$ is the random permutation of token indices shared between the watermarker and the detector. This is an optimal transport (OT) problem, which we solve efficiently using the Sinkhorn algorithm. The cost matrix $\texttt{C}\in\mathbb{R}^{N\times r}$ for the OT problem is defined as:
\begin{align}
    \texttt{C}_{i,j}=d\left(\frac{2\pi\cdot\Pi_t(i)}{N},\left(\frac{2\pi C_m(t)}{p}+\frac{2\pi j}{r}+\phi\right)\bmod 2\pi\right)
    \label{eq:cost_matrix}
\end{align}
for $(i,j)\in[0:N-1]\times[0:r-1]$ where $\Pi_t(i)$ denotes the permuted index of token $i$. The Sinkhorn algorithm returns the optimal joint distribution $Q^*_{X_t,Z_t}$, from which we extract the conditional:
\begin{align}
    Q^*_{X_t|Z_t=z_t}=\frac{Q^*_{X_t,Z_t}(X_t,z_t)}{P_{Z_t}(z_t)}=r\cdot Q^*_{X_t,Z_t}(X_t,z_t)
\end{align}
where the second equality uses $P(Z_t=z_t)=\frac{1}{r}$ since $V_t\sim\text{uniform}[0:r-1]$. Finally, the watermarker then samples token $X_t\sim Q^*_{X_t|Z_t=z_t}$ for the angle $z_t$ computed in the encoding stage \eqref{eq:zt-map}. This is the token received by the detector at token instance $t$.

\paragraph{Message decoding:} Upon receiving the token-key sequence $\{(x_t,s_t)\}_{t=1}^n$, the decoder estimates the embedded message $\hat{M}$ via minimum distance decoding over all possible codewords. 

Since the encoder transmits tokens as close as possible to the channel input $z_t$ (subject to the distortion-free constraint), the decoder first recovers the transmitted angles. At time $t$, using the received token $x_t$, shared key $v_t$, and permutation $\Pi_t$, the decoder estimates the encoder's channel input as:
\begin{align}
    \hat{z}_t = \frac{2\pi \cdot \Pi_t(x_t)}{N}
\end{align}
It then removes the shared randomness to recover the codeword symbol angle:
\begin{align}
    \hat{C}(t) = \left(\frac{2\pi \cdot \Pi_t(x_t)}{N} - \frac{2\pi v_t}{r}\right) \bmod 2\pi
\end{align}
which inverts the encoding operation in \eqref{eq:zt-map}.

For each candidate message $m\in\{0,1\}^k$, the decoder computes its angular codeword representation:
\begin{align}
    C_m^{\text{ang}} = \left[\frac{2\pi C_m(1)}{p}+\phi, \dotsc, \frac{2\pi C_m(n)}{p}+\phi\right]
\end{align}
and calculates the total distance to the received sequence:
\begin{align}
    D_m = \sum_{t=1}^n f\left(d\left(\hat{C}(t), C_m^{\text{ang}}(t)\right)\right) \label{eq:Dm_sum}
\end{align}
where $f$ is a non-decreasing function. The decoded message is then $\hat{M}=\arg\min_m D_m$.

We next ask whether \arcmark is optimal. That is, whether it embeds the maximum number of bits per token compatible with reliable decoding. To answer this, we first characterize the capacity of multi-bit watermarking in Sec.~\ref{cap}. In Theorem~\ref{cap_achieve}, we show that \arcmark achieves capacity (embeds the maximum number of bits per token) \textbf{under specific simplifying assumptions}.

\section{Capacity of Multi-bit Watermarking}\label{cap}
In this section, we derive the capacity of multi-bit watermarking. To define the capacity formally, we say a rate $R$, measured in bits/token, is \emph{achievable} if for every $\epsilon>0$, there exists a watermarker with $R=k/n$ and $P_e\le\epsilon$ under the distortion-free constraint. The \emph{capacity} $R_{\text{cap}}$ is the supremum over all achievable rates.

\textbf{Assumptions.} In order to complete our capacity formulation, we need to clarify our assumption on the next-token distribution $Q_{X_t}$. In practice, the watermarker knows the next-token distribution after the previous token has been generated, but it is computationally intractable to compute token distributions further in advance. For example, given $X_{1:t-1}$, running the LLM once gives the distribution $P_{X_t|X_{1:t-1}}$, but forming $P_{X_t,X_{t+1}|X_{1:t-1}}$ would require running the LLM many times, once for each possible value of $X_t$. Computing the distribution tokens further in the future would require exponentially many LLM computations. Thus, while in principle the watermarker has access to the complete joint distribution of a block of $n$ tokens, in practice it has very limited knowledge of token distributions beyond the next one. 

In order to capture this constraint in our model, we assume that the next-token distribution $Q_{X_t}$ is itself a random variable, which is not revealed to the watermarker until time $t$, and never revealed to the decoder. For further theoretical tractability, we assume that these distributions $Q_{X_t}$ are independent and identically distributed (i.i.d.) across time. We represent all i.i.d.\ $Q_{X_t}$ with $Q$, where $Q$ is a random variable in the simplex $\Delta_{\mathcal{X}}$ with the same distribution. We emphasize that this assumption that the token distributions are i.i.d.\ is necessary for the theoretical characterization of the capacity\footnote{The i.i.d. assumption allows us to simplify a $n$-time use of the LLM to a single-use, a methodology known as single-letterization in information theory.}, \textbf{and is not required for the performance of our practical scheme \arcmark}. Note that $Q$ is a random variable taking values in the simplex, and thus it is a randomly selected distribution on the token space. In particular, the fact that $Q_{X_t}$ are i.i.d. does \emph{not} mean that they are constant --- on the contrary, since each one is random, this assumption captures the fact that the next-token distributions do change from one token to the next.

The following theorem, proved in Appendix~\ref{proof_capacity}, characterizes the watermarking capacity. We interpret this theorem as follows. Let $X$ denote the output of the watermarked LLM and $W$ be an encoding of the message $M$ and the side information $S$. One can view this as a channel where the uncertainty is a result of the LLM's output distribution $Q_{X_t}$ at any time $t$. The capacity is the maximum information about $M$ that can be transmitted error-free and is captured by the maximal mutual information between $W$ and $X$ over all choices of the distribution of $W$ and a function $x(w,q)$ which determines the token based on the encoding of the message $w$ and the LLM distribution $q$. These choices are required to satisfy the distortion-free requirement that, given $q$, $X$ is distributed according to $q$.

\begin{theorem}\label{capacity}
    Assuming next-token distributions $Q_{X_t}$ are i.i.d., 
    the watermarking capacity is
    \begin{align}
        &R_{\textup{cap}}=\max_{P_W,x(w,q)}I(W;X),\nonumber\\ &\textup{s.t. }\Pr(X=x|Q=q)=q(x),
        \forall x\in\mathcal{X}, \forall q\in \Delta_\mathcal{X}\label{capacity_formula}
    \end{align}
    where $(W,Q) \sim P_W(w)P_Q(q)$, $X=x(W,Q)\in\mathcal{X}$, $P_W$ is a distribution on an arbitrary alphabet $\mathcal{W}$, and $Q$ is the random variable representing the token distribution from the simplex.
\end{theorem}

%Theorem~\ref{capacity} is proved in Appendix~\ref{proof_capacity}.

Next, we derive the capacity for a specific class of token distributions. Consider the class of distributions $\mathcal{P}_2(\mathcal{X})\subset \Delta_\mathcal{X}$, where any distribution $q\in \mathcal{P}_2(\mathcal{X})$ can be written as,
\begin{align}
        q(x)&=\begin{cases}
            1/2, & x=i,j\\
            0, & x\neq i,j
        \end{cases}
    \end{align}
for any pair of tokens $i,j\in\mathcal{X}$, $i\ne j$. Moreover, we assume $Q$ is uniformly distributed on $\mathcal{P}_2(\mathcal{X})$, i.e.,  $P(Q=q)=\frac{1}{\binom{N}{2}}$, $\forall q\in \mathcal{P}_2(\mathcal{X})$, where $N$ is the size of the token alphabet $\mathcal{X}$.

\begin{corollary}\label{dist_class}
   Assuming uniform token distributions over $\mathcal{P}_2(\mathcal{X})$, the capacity in Theorem~\ref{capacity} is 
    \begin{align}
        R_{\textup{cap}}=\log N+\sum_{t=1}^{N-1}\frac{t}{\binom{N}{2}}\log\frac{t}{\binom{N}{2}}.
    \end{align}
\end{corollary}

The proof of Corollary~\ref{dist_class} is given in Appendix~\ref{proof_dist_class}. In the limit as $N\to\infty$, the capacity for token distributions uniform on $\mathcal{P}_2(\mathcal{X})$ approaches $\approx0.2787\text{ bits/token}.$

The following theorem (proved in Appendix~\ref{proof_cap_ach}) shows that \arcmark achieves the capacity for the setting with binary token distributions as considered in Corollary~\ref{dist_class}.

\begin{theorem}\label{cap_achieve}
    Assuming token distributions are uniformly distributed on $\mathcal{P}_2(\mathcal{X})$, \arcmark achieves the capacity in Theorem~\ref{capacity} when:
    \begin{align}
        p=r=N,~
        \phi=\frac{\pi}{2N},~\text{and}~f(d)=-\log\left(1-\frac{d}{d_{\max}}\right)\text{ where }d_{\max}=\pi-\frac{\pi}{2N}.\label{cap_choice3}
    \end{align}
\end{theorem}

Beyond the connection between the capacity and \arcmark established by Theorem~\ref{cap_achieve}, there are a number of design aspects of \arcmark that take inspiration from the capacity result in Theorem~\ref{capacity}. We delve deeper into these connections in Appendix~\ref{appendix_connections}.

\section{Experimental Results}\label{sec:exp}

For baselines, consider inference-time watermarking methods that are LLM-agnostic, support message-agnostic detection, and have publicly available implementations. Within this setting, we evaluate \arcmark against two state-of-the-art multi-bit watermarking methods, BiMark~\cite{feng2025bimark} and MPAC~\cite{mpac}, in terms of message accuracy. Achieving distortion-free watermarking is the gold standard of the LLM watermarking problem, as it embeds messages without compromising text quality, making our problem setting strictly more challenging than that of non-distortion-free methods. We therefore include MPAC, a non-distortion-free method, only as a reference point, while BiMark serves as our primary baseline for all remaining comparisons, including robustness under paraphrasing attacks, distribution preservation via perplexity, downstream text quality, and zero-bit detection.

\textbf{Setup.} Three open-weight models are used for generation with temperature $1.0$ and top-50 sampling: Llama3-8B~\cite{dubey2024llama3}, Qwen3-8B~\cite{yang2025qwen3}, and Mistral-7B~\cite{jiang2023mistral}. We use the C4-RealNewslike dataset~\cite{raffel2020exploring} as prompts for generation. For all experimental results presented in this section, we use the identity function for $f$ (see \cref{eq:Dm_sum}) and set $\phi = 0$ (see \cref{eq:zt-map}), which provides a simple instantiation and yields performance comparable to the choice in Equation~\eqref{cap_choice3} based on our empirical observations. Across all experiments, we set $p = |\mathcal{M}|$, where $|\mathcal{M}|$ denotes the total number of possible messages, so that the number of symbols matches the number of messages and each embedded symbol can represent an entire message, and for the number of discrete side information values, we use $r = 64$. All error bars represent the standard error of the mean (SEM). 

\textbf{Side information generation.} In all experiments, for each prompt, we draw a 31-bit secret key uniformly at random from a PCG64 PRNG seeded deterministically by the prompt index plus a fixed offset. At each token position, the side information is computed as the SHA-256 hash of the secret key concatenated with the Unicode-normalized text of the three preceding tokens. The decoder, given the same key and the watermarked text, recomputes the side information identically at every position. This construction follows standard practice in the watermarking literature \cite{kirchenbauer2023watermark,tsurheavywater}.

\subsection{Message Accuracy}\label{subsec:message-acc}
Message accuracy measures whether the full embedded message is recovered exactly, and is the operationally relevant metric for watermarking applications: a decoded message is either correct or it is not. In contrast, bit accuracy (measuring the fraction of message bits that were recovered correctly) can be misleading, since recovering even a single bit incorrectly leads to a completely different message being decoded, potentially causing misattribution of the text to the wrong source. Furthermore, if the per-bit error probability is $p$ and bit-errors are independent, the probability of correct message recovery is $(1-p)^k$, which decays exponentially in the payload size $k$, making message accuracy a significantly harder objective than bit accuracy for larger payloads. \arcmark is designed with this objective in mind, and its advantage over prior methods becomes more pronounced as payload size grows.

Figure~\ref{fig:qwen3-msg-acc-by-text-length} illustrates this on Qwen3-8B for payloads of 2, 3, and 4 bytes. \arcmark achieves comparable or superior message accuracy across all settings, with the gap over prior methods widening at larger payloads. Results on Mistral-7B, Llama3-8B, and the 1-byte setting are deferred to Appendix~\ref{app:message-acc}, where the same trend holds.

\begin{figure}[t]
    \centering

    \begin{subfigure}{0.32\linewidth}
        \centering
        \includegraphics[width=\linewidth]{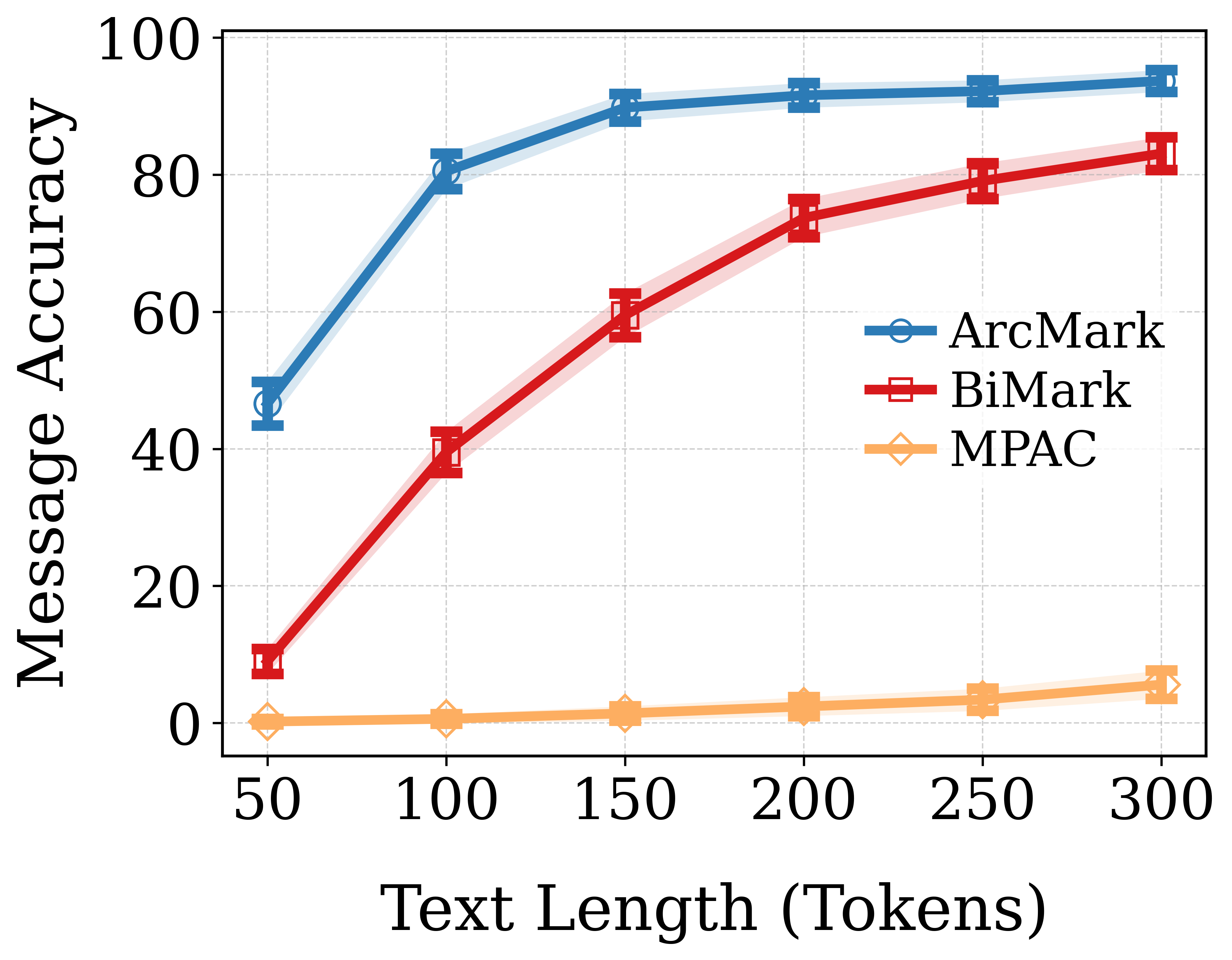}
        \label{fig:qwen3-16bit-acc}
    \end{subfigure}
    \hfill
    \begin{subfigure}{0.32\linewidth}
        \centering
        \includegraphics[width=\linewidth]{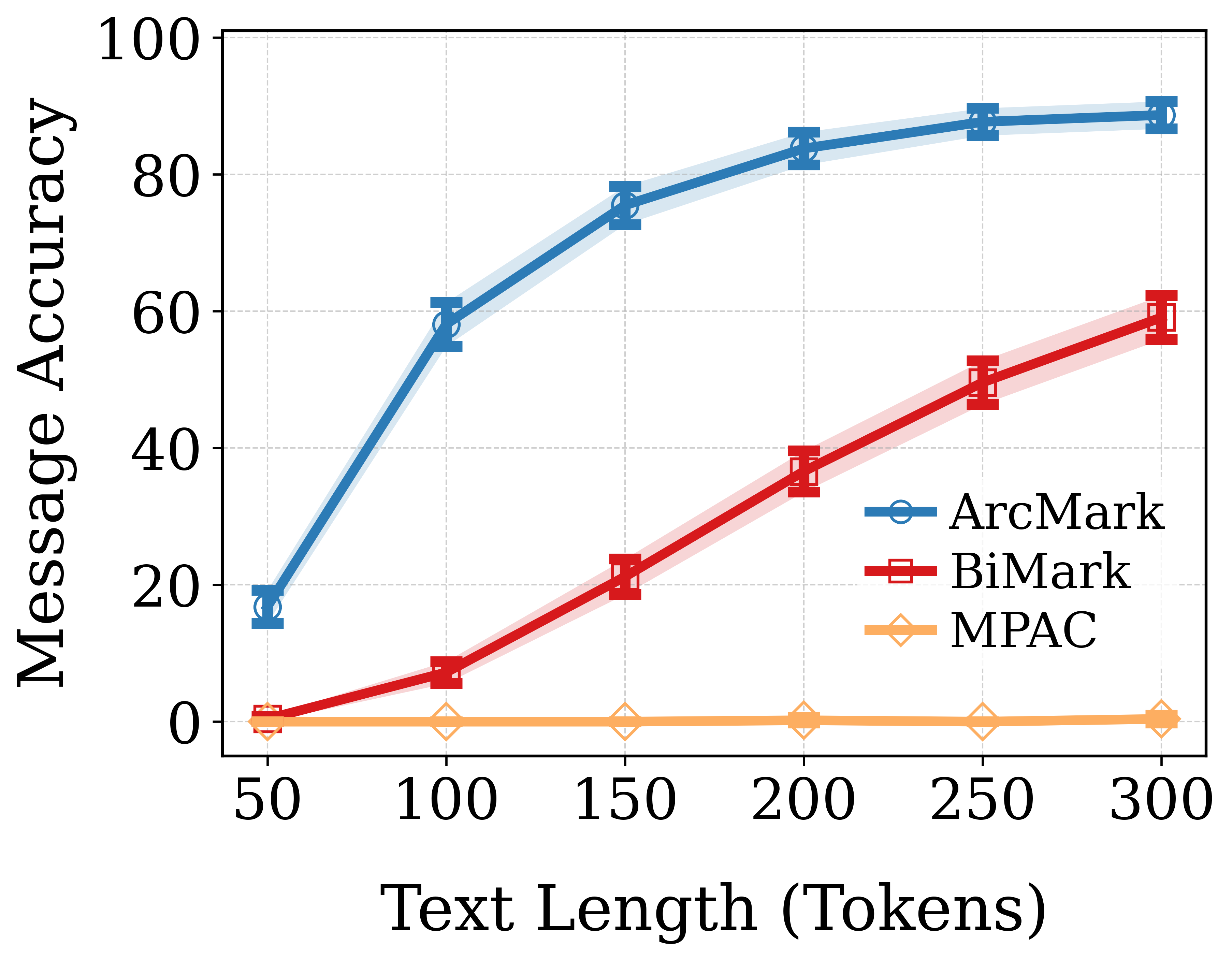}
        \label{fig:qwen3-24bit-acc}
    \end{subfigure}
    \hfill
    \begin{subfigure}{0.32\linewidth}
        \centering
        \includegraphics[width=\linewidth]{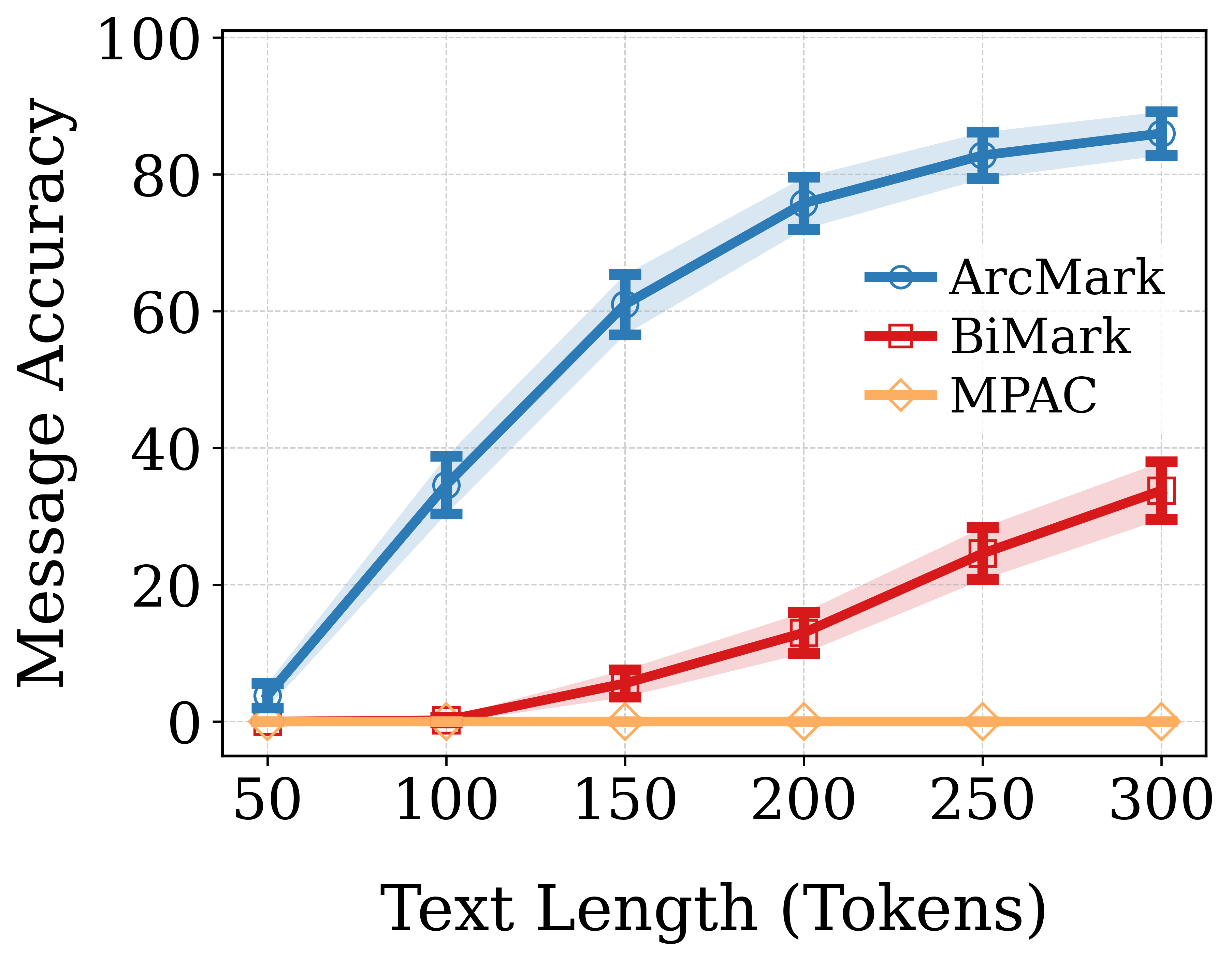}
        \label{fig:qwen3-32bit-acc}
    \end{subfigure}

    \caption{Message accuracy on Qwen3-8B for 2-byte (left), 3-byte (middle), and 4-byte (right) payloads, averaged over 1000 trials for 2 and 3 bytes and 500 trials for 4 bytes. Since recovering the full message implies that every bit has to be correct, the task becomes increasingly difficult as payload size grows. \arcmark is designed to directly optimize message recovery, which explains the growing performance gap over baselines.}
    \label{fig:qwen3-msg-acc-by-text-length}
\end{figure}

\begin{figure}
    \centering
    \includegraphics[width=\linewidth]{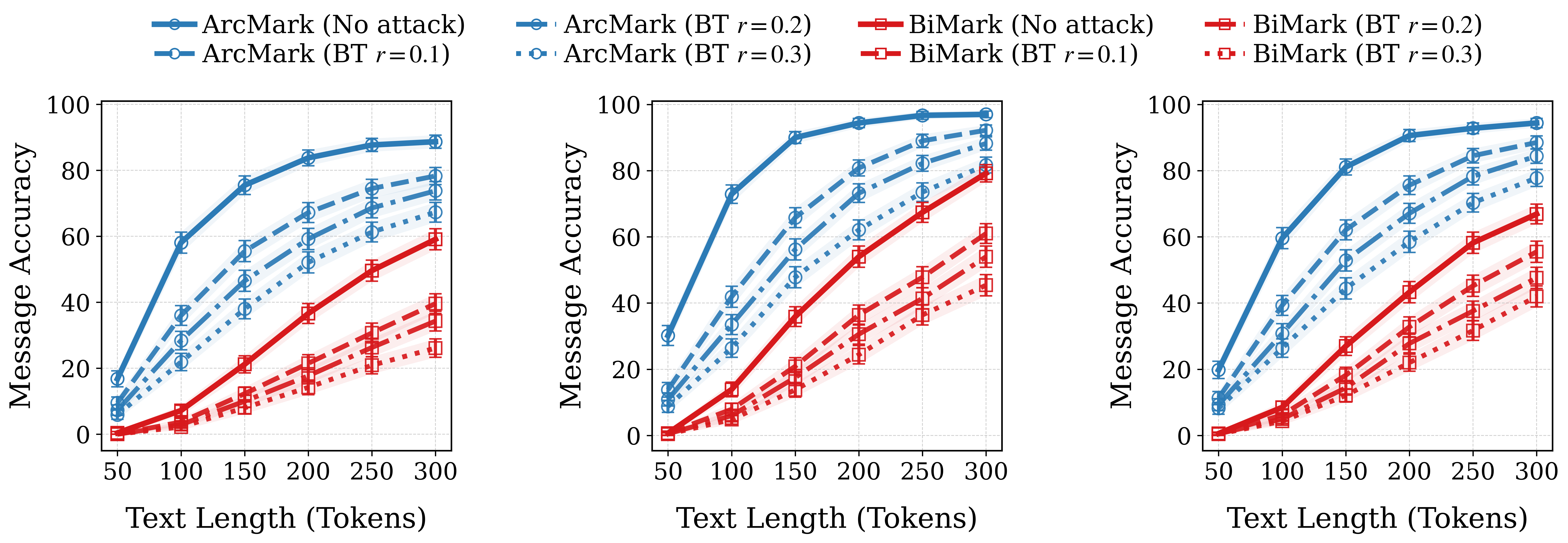}
\caption{Message accuracy under French back-translation (BT) paraphrasing attack at ratios $r \in \{0.1, 0.2, 0.3\}$ for a 3-byte payload, on Qwen3-8B (left), Llama3-8B (middle), and Mistral-7B (right), averaged over 1000 trials. Curves show the no-attack and attacked variants of \arcmark and BiMark. }
   \label{fig:french-bt-message-accuracy-24bit}
\end{figure}
% =========================
%%%%%%%%%%%%%%%%%%%%%%%%%%%%%%%%%%%%%%%%%%
\subsection{Robustness to Paraphrasing Attacks}\label{subsec:attack} 
We evaluate \arcmark's robustness against a back-translation (BT) paraphrasing attack, where each sentence is independently passed through an English $\to$ French $\to$ English translation pipeline using the Helsinki-NLP OPUS-MT models~\cite{tiedemann-thottingal-2020-opus}. The attack ratio $r$ denotes the probability with which each sentence is independently subjected to back-translation.
 
We evaluate at ratios $r \in \{0.1, 0.2, 0.3\}$ for a 3-byte payload on Llama3-8B, Qwen3-8B, and Mistral-7B. As baselines, we include the no-attack results for both \arcmark and BiMark. Figure~\ref{fig:french-bt-message-accuracy-24bit} shows that \arcmark maintains substantially higher message accuracy than BiMark under all attack ratios. Notably, \arcmark at $r=0.3$ still outperforms the unattacked BiMark on all three models. Results for a 2-byte payload are deferred to Appendix~\ref{app:attack}, where \arcmark similarly outperforms BiMark across all attack ratios.

%%%%%%%%%%%%%%%%%%%%%%%%%%%%%%%%%%
\subsection{Perplexity}\label{subsec:ppl} 
We use perplexity as a proxy for generation quality and to empirically assess whether \arcmark preserves the model's output distribution, which is guaranteed by construction (see equation~\eqref{ot}). Perplexity is computed on watermarked and non-watermarked text. \arcmark consistently achieves perplexity closer to the non-watermarked baseline than BiMark across all payload sizes and token lengths on both Qwen3-8B (Table~\ref{tab:qwen-ppl-no-attack}) and Llama3-8B (Table~\ref{tab:llama-ppl-no-attack} in Appendix~\ref{app:ppl}), providing empirical evidence for its distribution-preserving property.
\begin{table}[h]
\centering
\small
\setlength{\tabcolsep}{3.2pt}
\renewcommand{\arraystretch}{1.08}
\caption{Perplexity on Qwen3-8B for 2-, 3-, and 4-byte watermarks, averaged over 1000 trials for 2 and 3 bytes and 500 trials for 4 bytes. Bold indicates the method closer to the non-watermarked baseline.}
\label{tab:qwen-ppl-no-attack}
\resizebox{\linewidth}{!}{%
\begin{tabular}{@{}c c cc cc cc@{}}
\toprule
& & \multicolumn{2}{c}{2-byte} & \multicolumn{2}{c}{3-byte} & \multicolumn{2}{c}{4-byte} \\
\cmidrule(lr){3-4}
\cmidrule(lr){5-6}
\cmidrule(l){7-8}
Tokens & No watermark & ArcMark & BiMark & ArcMark & BiMark & ArcMark & BiMark \\
\midrule
50  & 4.661 $\pm$ 0.058 & \textbf{4.758 $\pm$ 0.058} & 4.967 $\pm$ 0.066 & \textbf{4.738 $\pm$ 0.057} & 4.976 $\pm$ 0.068 & \textbf{4.696 $\pm$ 0.080} & 4.965 $\pm$ 0.094 \\
100 & 4.390 $\pm$ 0.046 & \textbf{4.453 $\pm$ 0.047} & 4.624 $\pm$ 0.051 & \textbf{4.437 $\pm$ 0.048} & 4.630 $\pm$ 0.053 & \textbf{4.438 $\pm$ 0.067} & 4.618 $\pm$ 0.074 \\
150 & 4.241 $\pm$ 0.042 & \textbf{4.299 $\pm$ 0.043} & 4.463 $\pm$ 0.046 & \textbf{4.225 $\pm$ 0.044} & 4.435 $\pm$ 0.048 & \textbf{4.271 $\pm$ 0.062} & 4.493 $\pm$ 0.067 \\
200 & 4.142 $\pm$ 0.040 & \textbf{4.144 $\pm$ 0.041} & 4.345 $\pm$ 0.043 & \textbf{4.073 $\pm$ 0.042} & 4.308 $\pm$ 0.045 & \textbf{4.147 $\pm$ 0.060} & 4.393 $\pm$ 0.064 \\
250 & 4.041 $\pm$ 0.038 & \textbf{4.005 $\pm$ 0.040} & 4.241 $\pm$ 0.042 & \textbf{3.947 $\pm$ 0.040} & 4.216 $\pm$ 0.044 & \textbf{4.015 $\pm$ 0.058} & 4.299 $\pm$ 0.062 \\
300 & 3.941 $\pm$ 0.037 & \textbf{3.878 $\pm$ 0.039} & 4.148 $\pm$ 0.042 & \textbf{3.821 $\pm$ 0.039} & 4.116 $\pm$ 0.043 & \textbf{3.887 $\pm$ 0.056} & 4.189 $\pm$ 0.059 \\
\bottomrule
\end{tabular}%
}
\end{table}
%%%%%%%%%%%%%%%%%%%%%%%%%%%%%%%%%%%
\subsection{Quality on Downstream Tasks}\label{subsec:downstream} 
We evaluate the downstream utility of ArcMark under three downstream tasks: extractive question answering on SQuAD v2 \cite{qanda}, long-context Python code completion on LCC \cite{lcc}, and functional code correctness on HumanEval \cite{humaneval}. For each task, we compare ArcMark against an unwatermarked baseline and against BiMark on Llama3-8B and Qwen3-8B with 2-byte messages embedded in the generated text. Across all three tasks and both models, ArcMark closely matches the unwatermarked baseline, indicating that 2-byte watermarking imposes no meaningful degradation on downstream usability. Complete experimental details and results are reported in Appendix~\ref{app:downstream}.

%%%%%%%%%%%%%%%%%%%%%
\subsection{Zero-Bit Detection}\label{sec:zero-bit} 
Beyond message recovery, a watermarking system should also support zero-bit detection: given an arbitrary text, determine whether it was watermarked at all. The \arcmark decoder handles this naturally by thresholding its best-match score — the minimum distance between the observed token sequence and any valid codeword. A text is declared watermarked if this distance is sufficiently small, indicating that some codeword closely matches the observed sequence, and declared unwatermarked otherwise. This requires no modification to the encoder or watermarking process. We evaluate zero-bit detection on Llama3-8B, Qwen3-8B, and Mistral-7B, using watermarked texts as positives and human-written C4 articles as negatives, reporting true positive rate at a fixed false positive rate of 1\%. Full results are in Appendix~\ref{app:falsealarm}, where BiMark achieves modestly higher TPR at shorter text lengths, but both \arcmark and BiMark reach near-perfect detection accuracy from 150 tokens onward.
%%%%%%
\subsection{Ablation Study on Side Information Resolution}\label{subsec:ablation}

The parameter $r$ controls the number of discrete side information values and determines the size of the OT cost matrix (see Eq.~\ref{eq:cost_matrix}), directly affecting the computational cost of the Sinkhorn solver. We find that message accuracy is similar across $r \in \{64, 256, 512\}$, suggesting that a small value such as $r = 64$ is sufficient in practice. Full results are in Appendix~\ref{app:ablation}.

%%%%%%%%%%%%%%%%%%%%%%%%%%%%%%%%%%%%%
%\section{Broader Impact and Limitations}

\section{Conclusion}

Our results suggest a principled path forward for multi-bit watermark design: once watermarking is cast as channel coding with side information, novel schemes inspired by time-tested error-correcting codes can be %brought to bear
used to watermark LLMs. This approach led to \arcmark, which significantly outperforms existing benchmarks. Future work includes developing watermarking schemes that achieve capacity in a wider class of settings. Evaluating the capacity characterization in Theorem~\ref{capacity} for more realistic settings than in Corollary~\ref{dist_class} is a challenge; but doing so can yield insights leading to better watermarkers. Finally, an unexploited resource with our method is that the watermarker has \emph{feedback}, in that it knows the tokens that are sent to the decoder, which can be viewed as the output of the channel. While feedback typically does not increase capacity, it likely will improve probability of error or encoding complexity.

\textbf{Broader Impact.} Multi-bit watermarking of LLMs can curb AI misuse and add a layer of accountability, even beyond existing zero-bit watermarks. Multi-bit watermarks such as the one introduced in this paper can include information about the model version, user ID, or even the prompt, all of which are valuable for tracking and verifying LLM output. High-capacity low error-probability distortion-free watermarks will enable these societally-beneficial capabilities, all without impacting LLM performance. On the negative side, errors in message recovery can lead to misattribution, incorrectly linking a piece of text to the wrong source or user. By directly optimizing for message-level accuracy, \arcmark reduces this risk compared to methods that optimize only for bit-level accuracy.

\textbf{Limitations.} This work has several limitations that suggest directions for future research. First, our capacity analysis assumes specific distributions. Extending it to broader classes that better reflect deployed LLMs would give a fuller picture. More broadly, the information-theoretic framework developed here opens the door to more broader capacity characterizations and code constructions beyond random linear codes. Pursuing these directions would not only sharpen the theoretical foundations of multi-bit watermarking but also yield codes with higher recovery rates, stronger accuracy guarantees, and explicit robustness to adversarial edits such as text substitution and deletion.

% Multi-bit watermarking of LLMs can curb AI misuse and add a layer of accountability, even beyond existing zero-bit watermarkers. Multi-bit watermarkers such as the one introduced in this paper can include information about the model version, user ID, or even the prompt, all of which are valuable for tracking and verifying LLM output. High-capacity low error-probability distortion-free watermarkers will enable these societally-beneficial capabilities, all without impacting LLM performance, which means they are more likely to be used in practice.
% ============================================================
%  Bibliography
% ============================================================
\newpage
\printbibliography[heading=bibintoc, title={References}]

% ============================================================
%  Appendix
% ============================================================
\newpage
\appendix

%%%%%%%%%%%%%%%%%%%%%%%
\clearpage

\startcontents[appendix]

\section*{Appendix}
\printcontents[appendix]{}{1}{}

\clearpage
\section{Additional Related Work Details}\label{app:related-work}
We review the zero-bit and multi-bit watermarking methods most closely related to \arcmark.

\textbf{Zero-bit Watermarks.} Zero-bit watermarking schemes aim to determine whether a given sequence of outputs was generated by a model. Many multi-bit watermarking methods build directly on ideas developed in the zero-bit setting, with the Red-Green watermark playing a central role. We therefore begin by reviewing several notable zero-bit watermarking approaches.

The first watermark for LLMs was proposed by \cite{kirchenbauer2023watermark}, commonly referred to as the Red-Green watermark. This method partitions the token vocabulary into two disjoint lists and biases generation by exponentially tilting the model’s next-token distribution toward one of the lists. In practice, the Red-Green watermark is implemented by adding a bias term to the logits of green-list tokens.  Subsequent work extended this framework by introducing alternative schemes for generating the random side information used in the partitioning \cite{kirchenbauer2023reliability}. Since then, watermarking for LLMs has been studied extensively. 
Most notably, SynthID \cite{dathathri2024scalable} produces watermarked tokens via a strategy called tournament sampling. 
This watermark has been adopted in industry and covers various modalities, including text, images, and video generation.

\textbf{Cryptographic Approaches to Watermarking.} 
Cryptographic work has also studied watermarking under stronger adversarial assumptions. \cite{christ2024pseudorandomerrorcorrectingcodes} consider watermarking schemes that remain hidden from efficient adversaries without the key while retaining robustness to edits. Separately, \cite{pmlr-v247-christ24a} analyze zero-bit undetectable watermark detection under a cryptographic-style security notion and introduce entropy-based metrics for characterizing when detection is possible.

\textbf{Information-Theoretic Analysis of Watermarking.} 
Watermarking has a long history in information theory \cite{chen2000design,moulin2003information,martinian2005authentication}, particularly through formulations based on the Gelfand--Pinsker (GP) channel \cite{gel1980coding,villan2006text,willems2000informationtheoretical}. These classical approaches typically study watermarking of sequences via joint typicality arguments and assume perfect knowledge of the underlying source distribution at the decoder. In contrast, in LLM watermarking the next-token distribution is unknown to the detector. \cite{long2025optimized} analyzes text watermarking through the lens of hypothesis testing with side information, formalizing the fundamental trade-off between detection power and distortion in generated text. Their analysis yields the Correlated Channel Watermark, which constructs an optimal coupling between the side information shared with the detector and a random partition of the LLM vocabulary. Complementarily, \cite{he2024universally} characterize the universal Type-II error under worst-case control of the Type-I error by jointly optimizing the watermarking scheme and detector. \cite{he2025distframework} study multi-bit distributional embedding in a setting where the output distribution is effectively known in the decoding analysis, leading to an entropy-based rate. In this paper, we cast multi-bit LLM watermarking as a channel coding problem, and use information-theoretic insights to design optimal coding schemes. To our knowledge, this is the first formal capacity characterization for multi-bit LLM watermarking.

\textbf{Distortion-free Watermarking via Optimal Transport.} \cite{tsurheavywater} formulates LLM watermarking as an optimal transport problem between the original next-token distribution and a watermarked distribution conditioned on random side information, a perspective that is especially effective in the low-entropy regime where next-token predictions are near-deterministic. This formulation leads to distortion-free watermarking schemes such as HeavyWater and SimplexWater. In this paper, we extend the optimal-transport construction of \cite{tsurheavywater} to a circular (modulo) setting, enabling channel-coding ideas to be integrated with optimal transport. In addition, we also introduce an explicit message encoding/decoding scheme for efficient multi-bit watermarking, rather than applying \cite{tsurheavywater} repeatedly for each message symbol.

\textbf{Multi-bit Watermarks.} BiMark \cite{feng2025bimark} builds on unbiased green–red list watermarking but strengthens the signal by applying $k$ sequential, weak reweighting layers per token. Each layer uses an independent pseudorandom green/red partition and a fair coin flip to slightly bias the distribution, and the token is sampled from the final reweighted distribution. For detection, the verifier reconstructs the same $k$ partitions and flips and checks how often each generated token falls on the expected “green” side; these $k$ weak votes are then aggregated via majority voting to recover the embedded bit. In contrast to \arcmark, which decodes messages jointly from a series of tokens, BiMark performs per-token voting and aggregates these local decisions. Our joint decoding enables global, sequence-level strategies that can outperform token-by-token decoding.

Multi-bit Watermark via Position Allocation (MPAC) \cite{mpac} extends red–green watermarking to embed multiple bits by assigning each generated token to a pseudorandom message position and using the corresponding bit to choose which vocabulary subset (“colorlist”) to bias during sampling. At detection, the verifier computes a histogram over the colorlist for each message position, and determines the embedded message by determining which color sequence appears the most. The primary distinction between \arcmark and MPAC is that MPAC is not distortion-free. In contrast, \arcmark supports multi-bit watermarking without compromising text quality.

Other multi-bit watermarks include \cite{qupaper}, a non-distortion-free method based on pseudo-random segment assignment and \cite{xu2026xmarkreliablemultibitwatermarking}, a non-distortion-free scheme that builds an evergreen token set by intersecting green lists from multiple vocabulary permutations and uses a constrained token-shard mapping decoder for improved recovery at short text lengths. Among distortion-free methods, \cite{cui2026mc2markdistortionfreemultibitwatermarking} proposes a multi-bit framework using multi-layer sequential reweighting; \cite{jiang2026mirrormarkdistortionfreemultibitwatermark} embeds messages by mirroring sampling randomness in a measure-preserving manner; and \cite{stealthink} randomly categorizes the LLM vocabulary into message symbols and applies red-green reweighting. None of these distortion-free methods provides public implementations and we therefore cannot benchmark against them.

Achieving distortion-free watermarking can be considered the gold standard of the LLM watermarking problem, as it embeds messages without compromising text quality, making our problem setting strictly more challenging than that of non-distortion-free methods. We therefore benchmark  \arcmark primarily against BiMark~\cite{feng2025bimark}, and include MPAC~\cite{mpac} only as a reference point.
%%%%%%%%%%%%%%%%%%%%
\section{\arcmark Encoding and Decoding Algorithms}\label{app:alg}

The \arcmark encoding and decoding procedures described in Sec.~\ref{sec:arcmark} are summarized in Algorithms~\ref{alg1} and~\ref{alg2}, respectively.

\begin{algorithm}[h]
\caption{ARCMARK Encoding / Watermark Embedding}
\label{alg1}
\begin{algorithmic}[1]
\Require Message $m \in \{0,1\}^k$, side information $\{(G_t,v_t,\Pi_t)\}_{t=1}^n$, LLM next-token distributions $\{Q_{X_t}\}_{t=1}^n$, parameters $(p,r,\phi)$
\Ensure Watermarked token sequence $x_{1:n}$

\For{$t=1$ to $n$}
    \State Compute codeword symbol $C_m(t)=m\cdot G_t\text{ mod }p$
    \State Compute target angle
    \[
    z_t = \left(\frac{2\pi C_m(t)}{p} + \frac{2\pi v_t}{r} + \phi\right) \bmod 2\pi
    \]
    \State Form cost matrix $C^{(t)} \in \mathbb{R}^{N \times r}$ with entries
    \[
    C^{(t)}_{i,j}
    =
    d\!\left(
    \frac{2\pi \Pi_t(i)}{N},
    \left(\frac{2\pi C_m(t)}{p}+\frac{2\pi j}{r}+\phi\right)\bmod 2\pi
    \right)
    \]
    \State Solve the optimal transport problem
    \[
    Q^*_{X_t,Z_t}
    =
    \arg\min_{Q_{X_t|Z_t}}
    \mathbb{E}[d(2\pi \Pi_t(X_t)/N, Z_t)]
    \quad
    \text{s.t. }
    \mathbb{E}_{Z_t}[Q_{X_t|Z_t}] = Q_{X_t}
    \]
    \State Extract conditional watermarked distribution
    \[
    Q^*_{X_t|Z_t=z_t}(x) = r \, Q^*_{X_t,Z_t}(x,z_t)
    \]
    \State Sample token $x_t \sim Q^*_{X_t|Z_t=z_t}$
\EndFor
\State \Return $x_{1:n}$
\end{algorithmic}
\end{algorithm}

\begin{algorithm}[h]
\caption{ARCMARK Decoding}
\label{alg2}
\begin{algorithmic}[1]
\Require Received tokens $x_{1:n}$, side information $\{(v_t,\Pi_t)\}_{t=1}^n$, generator matrix $G$, message set $\mathcal{M}=\{0,1\}^k$, parameters $(p,\phi)$, nondecreasing score function $f$
\Ensure Decoded message $\hat{m}$

\For{$t=1$ to $n$}
    \State Estimate received token angle
    \[
    \hat{z}_t = \frac{2\pi \Pi_t(x_t)}{N}
    \]
    \State Remove shared randomness and estimate codeword angle
    \[
    \hat{C}(t)
    =
    \left(
    \frac{2\pi \Pi_t(x_t)}{N}
    -
    \frac{2\pi v_t}{r}
    \right)\bmod 2\pi
    \]
\EndFor

\ForAll{$m \in \mathcal{M}$}
    \State Compute codeword $C_m = mG$
    \State Form angular codeword representation
    \[
    C_m^{\mathrm{ang}}
    =
    \left[
    \frac{2\pi C_m(1)}{p}+\phi,\dots,
    \frac{2\pi C_m(n)}{p}+\phi
    \right]
    \]
    \State Compute distance score
    \[
    D_m = \sum_{t=1}^n f\!\left(d\!\left(\hat{C}(t), C_m^{\mathrm{ang}}(t)\right)\right)
    \]
\EndFor

\State Decode by minimum-distance rule
\[
\hat{m} = \arg\min_{m \in \mathcal{M}} D_m
\]
\State \Return $\hat{m}$
\end{algorithmic}
\end{algorithm}
%%%%%%%%%%%%%%%%%%%%%
\section{Proof of Theorem~\ref{capacity}}\label{proof_capacity}

In this proof we use the vector notation $X^n=(X_1,X_2,\ldots,X_n)$.

\emph{Achievability:} We first prove that the capacity is at least equal to the quantity given in \eqref{capacity_formula}. Consider any discrete distribution $P_W$ and function $x(w,q)$ satisfying the condition in the capacity expression in \eqref{capacity_formula}; that is,
\begin{equation}\label{theorem_requirement}
    \Pr(X=x|Q=q)=q(x),\quad \forall x\in\calX,\forall q\in\Delta_{\calX}
\end{equation}
where $(W,Q)\sim P_W(w)P_Q(q)$, $X=x(W,Q)$. We synthesize a channel from $U$ to $X$ with common information $S$, where $U\in[0,1]$. By classical results in information theory (see \cite{elgamal2011network}, Section 7.4), we know that the rate $I(U;X|S)$ is achievable. We will show that this is equal to $I(W;X)$. First, assume that $\calS=[0,2\pi]$, and $P_S$ is uniform on this interval. Define a function $f:[0,2\pi]\to \calW$ such that, if $U\sim\text{Unif}[0,2\pi]$, then $f(U)$ is distributed according to $P_W$ (this is possible for any discrete distribution $P_W$). Also let $\oplus$ denote addition modulo $2\pi$. Let $W=f(U\oplus S)$, and then $X=x(W,Q)$.

We first show that this scheme is in fact distortion-free, as defined in \eqref{distortion-free}. Consider any $u\in[0,2\pi]$, and any $q\in\Delta_{\calX}$. Then, given any $u$, 
\begin{equation}
\Pr(X=x|U=u,Q=q)=\Pr(X=x|W=f(u\oplus S))=\sum_w P_W(w)\mathbf{1}_{\{x(w,q)=x)\}}=q(x)
\end{equation}
where we have used the fact that, for any $u$, $u\oplus S$ is uniformly distributed on $[0,2\pi]$, and so $f(u\oplus S)$ has the same distribution as $W$. Then the last equality holds by the assumption in \eqref{theorem_requirement}. This shows that the resulting code will satisfy the distortion-free requirement in \eqref{distortion-free}.

Now we prove that $I(U;X|S)=I(W;X)$. Note that $S$ is independent of $U\oplus S$. Since $Q$ is independent of the pair $(S,U)$, the three variables $S$, $U\oplus S$, and $Q$, are mutually independent. Since $X$ depends only on  $W=f(U\oplus S)$ and $Q$, it must be that $X$ is independent of $S$. Thus $H(X|S)=H(X)$. Moreover, since $W=f(U\oplus S)$, and $(U,S)\to W\to X$ is a Markov chain, we have
\begin{equation}
H(X|S,U)
=H(X|W).
\end{equation}
Thus we may write
\begin{align}
I(U;X|S)
&=H(X|S)-H(X|S,U)
\\&=H(X)-H(X|W)
\\&=I(W;X).
\end{align}

\emph{Converse:} Now we show that the capacity is no greater than the quantity given in \eqref{capacity_formula}. Consider any achievable rate $R$, where $R=\frac{k}{n}$. By the definition of achievability, there is a sequence of codes, one for each length $n$, each with rate $R$, and probability of error $P_e^{(n)}$, where $P_e^{(n)}\to 0$ as $n\to\infty$. We denote the message as $M$, which is selected uniformly from $\{0,1\}^k$. At time-step $i$, the token is denoted $X_i$, the shared secret is denoted $S_i$, and the LLM distribution is denoted $Q_i$. We allow the watermarker access to private randomness in generating its next token, which we denote $Z_i$. Thus, the watermarker decides $X_i$ based on $M$, $S^i$, $X^{i-1}$, $Z_i$, and $Q_i$. We define $W_i=(M,S^i,X^{i-1},Z^i$. Thus, $X_i$ is a deterministic function of $W_i$ and $Q_i$; we denote this function $x_i(w_i,q_i)$. Since the message $M$ is decoded from $X^n$ and $S^n$, by Fano's inequality,
\begin{equation}
    H(M|X^n,S^n)\le P_e^{(n)} nR=n\eps_n
\end{equation}
where $\eps_n=P_e^{(n)}R$, which goes to $0$ as $n\to\infty$. We now have the chain of inequalities
\begin{align}
nR&=H(M)\label{converse0}
\\&=I(M;X^n,S^n)+H(M|X^n,S^n)
\\&\le I(M;X^n,S^n)+n\eps_n\label{converse1}
\\&=\sum_{i=1}^n I(M;X_i,S_i|X^{i-1},S^{i-1})+n\eps_n
\\&=\sum_{i=1}^n [I(M;S_i|X^{i-1},S^{i-1})+I(M;X_i|X^{i-1},S^i)]+n\eps_n
\\&=\sum_{i=1}^n I(M;X_i|X^{i-1},S^i)+n\eps_n\label{converse2}
\\&\le \sum_{i=1}^n I(M,S^i,X^{i-1},Z_i;X_i)+n\eps_n
\\&=\sum_{i=1}^n I(W_i;X_i)+n\eps_n\label{converse3}
\end{align}
where \eqref{converse0} follows since $M$ is uniformly distributed on $\{0,1\}^k$ where $k=nR$, \eqref{converse1} follows from the above application of Fano's inequality, \eqref{converse2} holds since $S_i$ is independent of $(M,X^{i-1},S^{i-1})$, and in \eqref{converse3} we have used the definition of $W_i$. Recalling the definition of $x_i(w_i,q_i)$, for any $q\in\Delta_{\calX}$ and any $x\in\calX$,
\begin{align}
\Pr(x_i(W_i,q)=x)
&=\sum_{w_i:x_i(w_i,q)=x} P_{W_i}(w_i) 
\\&=\sum_{m,s^{i-1},x^{i-1}} P_{M,S^{i-1},X^{i-1}}(m,s^{i-1},x^{i-1})\nonumber 
\\&\qquad
\cdot\Pr_{S_i,Z_i}(X_i=x|W=m,X^{i-1}=x^{i-1},Q_i=q)
\\&=q(x)
\end{align}
where the last equality follows from the distortion-free requirement of the watermarker in \eqref{distortion-free}.
Thus, for each $i$, $P_{W_i}$ and $x_i(w_i,q_i)$ satisfy the condition in the theorem statement in \eqref{theorem_requirement}. This means that
\begin{equation}
nR\le n \left(\max_{P_W,x(w,q)} I(W;X)\right)+n\eps_n
\end{equation}
where again this max is over $P_W,x(w,q)$ satisfying \eqref{theorem_requirement}. Dividing by $n$ and taking a limit as $n\to\infty$ proves the converse bound.
%%%%%%%%%%%%%%%%%%%%%
%%%%%%%%%%%%%%%%%%%%%%%%
\section{Proof of Corollary~\ref{dist_class}}\label{proof_dist_class}
    Consider any distribution $P_W$ and function $x(w,q)$ satisfying the condition in Theorem~\ref{capacity}. Note that
    \begin{equation}
        P(X=x)=\sum_{q\in \Delta_{\calX}}P(Q=q)P(X=x|Q=q)=\sum_{q\in \Delta_{\calX}}P(Q=q)q(x)
    \end{equation}
    where the second equality follows from the condition in Theorem~\ref{capacity} that $P(X=x|Q=q)=q(x)$. That is, the distribution of $X$ is fixed by the problem setup, and unaffected by the optimization over $P_W$ and $x(w,q)$. Thus,

    \begin{align}\label{cap_simpl}
        C=\max_{P_W,x(w,q)}I(W;X)=H(X)-\min_{P_W,x(w,q)}H(X|W)
    \end{align}
    For $Q$ uniform on $\calP_2(\calX)$, each token is equally likely, which means $P(X=x)=\frac{1}{N}$, so
    \begin{equation}
        H(X)=\log N.
    \end{equation}
    Now consider,
    \begin{align}
        P(X=x|W=v)&=\sum_{q\in\Delta_{\calX}} P(Q=q)P(X=x|Q=q,W=w)
        \\&=\sum_{q\in\Delta_{\calX}} P(Q=q)\mathbf{1}_{\{x(q,w)=x\}}.
    \end{align}
    Recall that for $\calP_2(\calX)$, $q$ takes on $\binom{N}{2}$ different distributions, each uniform on two tokens. Let $q_{i,j}$ be the distribution on tokens $i,j$ each having probability $1/2$. By the requirement of the capacity expression $x(q_{i,j},w)$ can only be either $i$ or $j$. Thus, we can write
    \begin{equation}
        P(X=x|W=w)=\frac{1}{\binom{N}{2}}\sum_{i\in\calX\setminus\{x\}} \mathbf{1}_{\{x(q_{i,x},w)=x\}}.\label{conditional_distribution}
    \end{equation}
    Recall that
    \begin{equation}
        H(X|W)=-\sum_{w\in\calW} P(W=w) \sum_{x\in\calX} P(X=x|W=w)\log P(X=x|W=w).
    \end{equation}
    To minimize $H(X|W)$, we need to minimize the entropy of the conditional distribution given in \eqref{conditional_distribution} for each $w$.

\begin{table}[t]
\centering
\begin{subtable}[t]{0.48\textwidth}
  \centering
  \begin{tabular}{|c|c|c|c|}
    \hline
       &  $q_{1,2}$& $q_{1,3}$& $q_{2,3}$\\
       \hline
       $w_1$  & & &\\
       \hline
       $w_2$  & & &\\
       \hline
       $\vdots$  & $\vdots$ & $\vdots$ & $\vdots$\\
       \hline
       $w_d$  & & &\\
       \hline
  \end{tabular}
  \caption{3-token example: setting}
  \label{tab:example-left}
\end{subtable}\hfill
\begin{subtable}[t]{0.48\textwidth}
  \centering
  \begin{tabular}{|c|c|c|c|}
    \hline
       &  $q_{1,2}$& $q_{1,3}$& $q_{2,3}$\\
       \hline
       $w_1$  & $1$ & $1$ & $2$\\
       \hline
       $w_2$ & $2$  & $3$ & $3$\\
       \hline
  \end{tabular}
  \caption{3-token example: optimal solution}
  \label{tab:example-right}
\end{subtable}

\caption{3 token example.}
\label{3-token example}
\end{table}

Consider an example with $|\mathcal{X}|=3$, with $\calW=\{w_1,\ldots,w_d\}$, where $W$ is uniformly distributed on $\calW$. Minimizing $H(X|W)$ is equivalent to filling Table~\ref{tab:example-left} with token indices $\{1,2,3\}$ such that each column $q_{i,j}$, has to be filled with half token $i$ and half token $j$ to satisfy the distortion-free property, while minimizing the entropy of each row to minimize $H(X|W)$. To minimize the entropy of each row in Table~\ref{tab:example-left}, we want to maximize the use of the same index in as many columns as possible. However, one index can only appear in two out of the three columns, as column $q_{i,j}$ can only have indices $i$ or $j$. Thus, the optimal setting is when each row has two of the same index and one new index, as shown in Table~\ref{tab:example-right}. Thus, for this example, we have,
\begin{align}
    \min_{P_W, \ x(w,q)}H(X|W)&=H\left(\frac{2}{3},\frac{1}{3}\right)\\
    &=-\frac{2}{3}\log\frac{2}{3}-\frac{1}{3}\log\frac{1}{3}\\
    &=\log 3-\frac{2}{3}.
\end{align}
Thus, the capacity is
\begin{equation}
    C=\log 3 - \left(\log 3-\frac{2}{3}\right)=\frac{2}{3}.
\end{equation}

Next, we generalize this argument to arbitrary $|\mathcal{X}|=N$.

\begin{claim}\label{cl1}
For token distributions uniform on $\calP_2(\calX)$ with $\calX=\{1,\ldots,N\}$,
\begin{equation}
    \min_{P_W, x(q,w)} H(X|W) = \sum_{t=1}^{N-1}\frac{t}{\binom{N}{2}}\log\frac{t}{\binom{N}{2}}.
\end{equation}

\end{claim}
\begin{proof}[Proof of Claim~\ref{cl1}]

    We have the lower bound
    \begin{equation}
        H(X|W)=\sum_{w\in\calW} P_W(w)H(X|W=w)\ge \min_{w\in\calW} H(X|W=w).
    \end{equation}
    To further lower bound $H(X|W=w)$, we consider row for $W=w$ in the equivalent of Table~\ref{tab:example-left} for arbitrary $N$, and we prove a lower bound on the entropy of this row. 
 
Let $s_i$ denote the number of times the token index $i\in\{1,\dotsc,N\}$ is repeated in row $w$ of the table. WLOG, assume that the indices are arranged such that $s_1\geq s_2\geq\dotsc\geq s_N$. For every $t\in\{1,\dotsc,N\}$, we have,
\begin{align}
    \sum_{i=1}^ts_i
    &\le |\{(i,j)\in \{1,\ldots,N\}^2:i<j,\ i\le t\text{ or }j\le t\}|\label{upper_bound1}
    \\&=|\{(i,j)\in \{1,\ldots,N\}^2:i<j,\ i\le t,\ j>t\}|
    +|\{(i,j)\in \{1,\ldots,N\}^2:i<j,\ i\le t,\ j\le t\}|
    \\&= t(N-t)+\binom{t}{2}\label{upper_bound}
\end{align}
where \eqref{upper_bound} follows from the fact that the number of appearances of tokens $i\le t$ is upper bounded by the number of columns that are labeled by an index that is less than or equal to $t$.

In each of the $q_{i,j}$ columns, note that each token index $i$ only appears in $N-1$ columns, as $i$ can only pair up with the $N-1$ other indices. Now we define $s_1^*=N-1,s_2^*=N-2,\dotsc,s_{N-1}^*=1,s_N^*=0$. Observe that
\begin{align}
    \sum_{i=1}^ts_i^*=\sum_{i=1}^t(N-i)=t(N-t)+\binom{t}{2}
\end{align}
which achieves the upper bound in \eqref{upper_bound}. Therefore,
\begin{align}
\sum_{i=1}^ts_i\leq\sum_{i=1}^ts_i^*\quad\text{and}\quad\sum_{i=1}^Ns_i=\sum_{i=1}^Ns_i^*=\binom{N}{2}
\end{align}
Since 
%$H(B_\ell)=-\sum_{i=1}^N\frac{s_i}{\binom{N}{2}}\log\frac{s_i}{\binom{N}{2}}$ and 
$\phi(x)=-x\log x$ is concave,
\begin{align}
H(X|W=w)&=H\left(\frac{s_1}{\binom{N}{2}},\dotsc,\frac{s_N}{\binom{N}{2}}\right)
\\&\ge
    H\left(\frac{s_1^*}{\binom{N}{2}},\dotsc,\frac{s_N^*}{\binom{N}{2}}\right)
\\&=-\sum_{t=1}^{N-1} \frac{t}{\binom{N}{2}}\log \frac{t}{\binom{N}{2}},
\end{align}
where the inequality follows from Karamata's inequality. This proves that the required quantity is a lower bound on $H(X|W)$.

It remains to prove that there exists a construction that achieves the same value for $H(X|W)$. To do this, we consider a binary alphabet $\calW=\{w_1,w_2\}$, where $W$ is uniformly distributed between these two letters. For $w_1$, let $x(q_{i,j},w_1)$ be the smaller of $i,j$, and for $w_2$, let $x(q_{i,j},w_2)$ be the larger of $i,j$. It is easy to see that with this construction, for each row of the table, $s_i=s_i^*$: For example, in the $w_1$ row, the number of times $i$ appears is equal to the number of $j>i$, which is $N-i=s_i^*$. Thus it achieves the same entropy as above. This construction also satisfies the property that the column for $q_{i,j}$ contains half $i$ and half $j$, it satisfies the distortion-free property. This completes the proof.
\end{proof}
%%%%%%%%%%%%%%%
\section{Proof of Theorem~\ref{cap_achieve}}\label{proof_cap_ach}

Using the particularizations of \arcmark in \eqref{cap_choice3}, we can write $z_t$ as
\begin{equation}\label{cap_achieve_zt}
    z_t=\frac{2\pi}{N}\left(C_m(t)+v_t+\frac{1}{4}\right)\text{ mod }2\pi
\end{equation}
where $C_m(t)$ is the codeword symbol for message $m$ at time $t$, $v_t$ is part of the shared information, chosen uniformly at random from $\{0,\ldots,N-1\}$. Also $C_m^{\text{ang}}$ becomes
\begin{equation}
    C_m^{\text{ang}}(t)=\frac{2\pi}{N}\left(C_m(t)+\frac{1}{4}\right).
\end{equation}
Finally, the decoding rule is
\begin{equation}\label{alt_decoding_rule}
    \hat{M}=\arg\min_m \sum_{t=1}^n -\log\left(1-\frac{d(\hat{C}(t),C_m^{\text{ang}}(t))}{d_{\max}}\right).
\end{equation}
Note that $d_{\max}=\pi-\frac{\pi}{2N}$ is the largest possible value of the distance $d$ using this construction.

In \arcmark, each token index $i=\{1,\dotsc,N\}$ at time $t$ is represented by an angle $A_i=\frac{\Pi_t(i)}{N}2\pi$. If $Q=q_{i,j}$, the uniform binary distribution on tokens $i$ and $j$, then the solution to the optimization in \eqref{ot} is deterministic, given by
\begin{align}
    P(X_t=x|z_t,q_{i,j})&=\begin{cases}
        \mathbf{1}_{\left\{d\left(z_t,A_i\right)< d\left(z_t,A_j\right)\right\}},&x=i \\
        \mathbf{1}_{\left\{d\left(z_t,A_i\right)> d\left(z_t,A_j\right)\right\}},&x=j \\
        0, & x\neq i,j.
    \end{cases}
\end{align}
Note that because of the added $1/4$ term in \eqref{cap_achieve_zt}, there is never a tie between the two distances. Let $u_t=C_m(t)+v_t\text{ mod }N$, so $z_t=\frac{2\pi}{N}(u_t+1/4)\text{ mod }2\pi$. Given $u_t$, we can list the points representing tokens on the circle in decreasing order by distance to $z_t$, as follows:
\begin{align}
    \frac{2\pi}{N} u_t,\ 
    \frac{2\pi}{N}(u_t+1)\text{ mod }2\pi,\ 
    \frac{2\pi}{N}(u_t-1)\text{ mod }2\pi,\ 
    \frac{2\pi}{N}(u_t+2)\text{ mod }2\pi,\ 
    \ldots,\ 
    \frac{2\pi}{N}\left(u_t-\left\lfloor\frac{N}{2}\right\rfloor\right)\text{ mod }2\pi.
\end{align}
Let $y_1,y_2,\ldots,y_N$ be the tokens such that these points are $A_{y_1},A_{y_2},\ldots,A_{y_N}$ respectively. Note that, for $a\in\{1,\ldots,N\}$,
\begin{equation}\label{distance_a}
    d(z_t,A_{y_a})=\frac{\pi(2a-1)}{2N}.
\end{equation}
As expected, the $y_a$ are in order by distance to $z_t$.

Thus, given $Q=q_{i,j}$ where $i=y_a$ and $j=y_b$, the optimal transport solution is simply to select $i$ if $a<b$, and $j$ if $a>b$. Thus, for any $a\in\{1,\ldots,N\}$, $X_t=y_a$ if $i=y_a$ and $j=y_b$ for any $b>a$. Thus,
\begin{align}\label{likelihood_fnc}
    P(X_t=y_a|z_t)=\frac{N-a}{\binom{N}{2}}.
\end{align}
From \eqref{distance_a}, we can write
\begin{equation}
    a=\frac{N}{\pi}\left(d(z_t,A_{y_a})+\frac{\pi}{2N}\right)
\end{equation}
so
\begin{align}
    P(X_t=y_a|z_t)&=\frac{1}{\binom{N}{2}}\left(N-\frac{N}{\pi}\left(d(z_t,A_{y_a})+\frac{\pi}{2N}\right)\right)\\
    &=\frac{N}{\pi \binom{N}{2}} \left(\pi-\frac{\pi}{2N}-d(z_t,A_{y_a})\right)\\
    &=\frac{N}{\pi \binom{N}{2}} \left(d_{\max}-d(z_t,A_{y_a})\right).
\end{align}
Since this holds for all $a$, and $y_a$ runs over all tokens in the alphabet for all $a$, this means that for any token $x$, we can rewrite the likelihood function in terms of the distance function:
\begin{equation}
    P(X_t=x|z_t)=\frac{N}{\pi \binom{N}{2}} \left(d_{\max}-d(z_t,A_x)\right).
\end{equation}
Recalling that $A_x=\frac{2\pi \Pi_t(x)}{N}$, we have
\begin{align}
    d(z_t,A_{x_t})&=d\left(C_m^{\text{ang}}+\frac{2\pi v_t}{N}\text{ mod }2\pi, \frac{2\pi \Pi_t(x_t)}{N}\right)
    \\&=d\left(C_m^{\text{ang}},\frac{2\pi}{N}\left(\Pi_t(x_t)-v_t\right)\text{ mod }2\pi\right)
    \\&=d(C_m^{\text{ang}},\hat{C}(t))
\end{align}
where we have used the fact that
\begin{equation}
    \hat{C}(t)=\frac{2\pi}{N}\left(\Pi_t(x_t)-v_t\right)\text{ mod }2\pi.
\end{equation}
Thus,
\begin{equation}
    P(X_t=x_t|z_t)=\frac{N}{\pi \binom{N}{2}} \left(d_{\max}-d(C_m^{\text{ang}},\hat{C}(t))\right).
\end{equation}
Considering this likelihood function for all time instances $t$, we have
\begin{equation}
    P(X^n=x^n|z^n)=\prod_{t=1}^n \frac{N}{\pi \binom{N}{2}} \left(d_{\max}-d(C_m^{\text{ang}},\hat{C}(t))\right).
\end{equation}
Therefore, the maximum likelihood decoder is identical to the decoding rule given in \eqref{alt_decoding_rule}.

By classical results in information theory (see \cite{polyanskiy2025information}, Theorem~18.13), random linear coding as in \eqref{linear_code} with maximum likelihood decoding can achieve arbitrarily small probability of error for any rate below the mutual information of the channel, assuming the input distribution is uniform. In this case, the channel has input given by the code sequence $C_m(t)$, the output is the token $X_t$, and we have side information $W_t$, known to both encoder and decoder (so we can condition on $W_t$ in the mutual information). Based on the calculation of the likelihood function in \eqref{likelihood_fnc},
\begin{align}
    I(C(t);X_t|W_t)&=H(X_t|W_t)-H(X_t|C(t),W_t)
    \\&=H(X_t)-H(X_t|Z_t)
    \\&=\log N+\sum_{i=1}^{N-1} \frac{i}{\binom{N}{2}}\log \frac{i}{\binom{N}{2}}.
\end{align}
Since this quantity matches the capacity found in Corollary~\ref{dist_class}, this proves that \arcmark achieves capacity for this setting.
%%%%%%%%
%%%%%%%%
\section{Connections Between \arcmark and the Capacity Result}\label{appendix_connections}

The capacity result of Theorem~\ref{capacity} applies given certain assumptions on the underlying token distributions, while \arcmark is a scheme that can be applied to any LLM even if these assumptions are not satisfied. However, several design aspects of \arcmark take inspiration from Theorem~\ref{capacity}.

First, \arcmark uses random linear codes because they are known to achieve capacity for channels where the optimal input distribution is uniform, as it is in this case.

Secondly, the achievability part of the proof of Theorem~\ref{capacity} in Appendix~\ref{proof_capacity} introduces a specific achievable watermarking scheme that, while not practical as \arcmark, shares some common characteristics. This theoretical scheme uses a random variable $U$ to represent the ``input'' to a channel, and $S$ to represent the side information. These variables are in the interval $[0,2\pi]$, and are then added together with addition modulo $2\pi$; this sum is then further processed to create the watermarked token. This is identical to the operations using points on a circle in \arcmark. (In \arcmark, the shift $\frac{2\pi v_t}{r}$ is added to the point on the circle in \eqref{eq:zt-map}; in the theoretical scheme this value is denoted as $S$.) The main difference is that in \arcmark, these symbols are taken from finite sets, rather than the continuous interval as in the theoretical scheme. Using finite sets makes the scheme more practical, but requires further design choices: Namely, the use of the random mapping $Q_{X_t|Z_t}$ is necessary because $Z_t$ only takes values in a finite set. If it took values on the continuous interval as in the achievability proof, the mapping could be deterministic while yielding any desired token distribution, as in the use of  the $f$ function in the proof. However, this is not possible with values in a finite set. Optimal transport naturally allows us to ensure the correct token distribution while minimizing the distance to the codeword.

%%%%
%%%%%%%%%
\section{Additional Message Accuracy Results}\label{app:message-acc}
Figure~\ref{fig:llama3-mistral-msg-acc-by-text-length} reports results on Llama3-8B and Mistral-7B for 2-, 3-, and 4-byte payloads. Figure~\ref{fig:message-acc-8bit} shows message accuracy for a 1-byte payload on Llama3-8B, Qwen3-8B, and Mistral-7B. The trends are consistent with those reported in Section~\ref{subsec:message-acc}.
\begin{figure}[t]
    \centering

    \begin{subfigure}{0.32\linewidth}
        \centering
        \includegraphics[width=\linewidth]{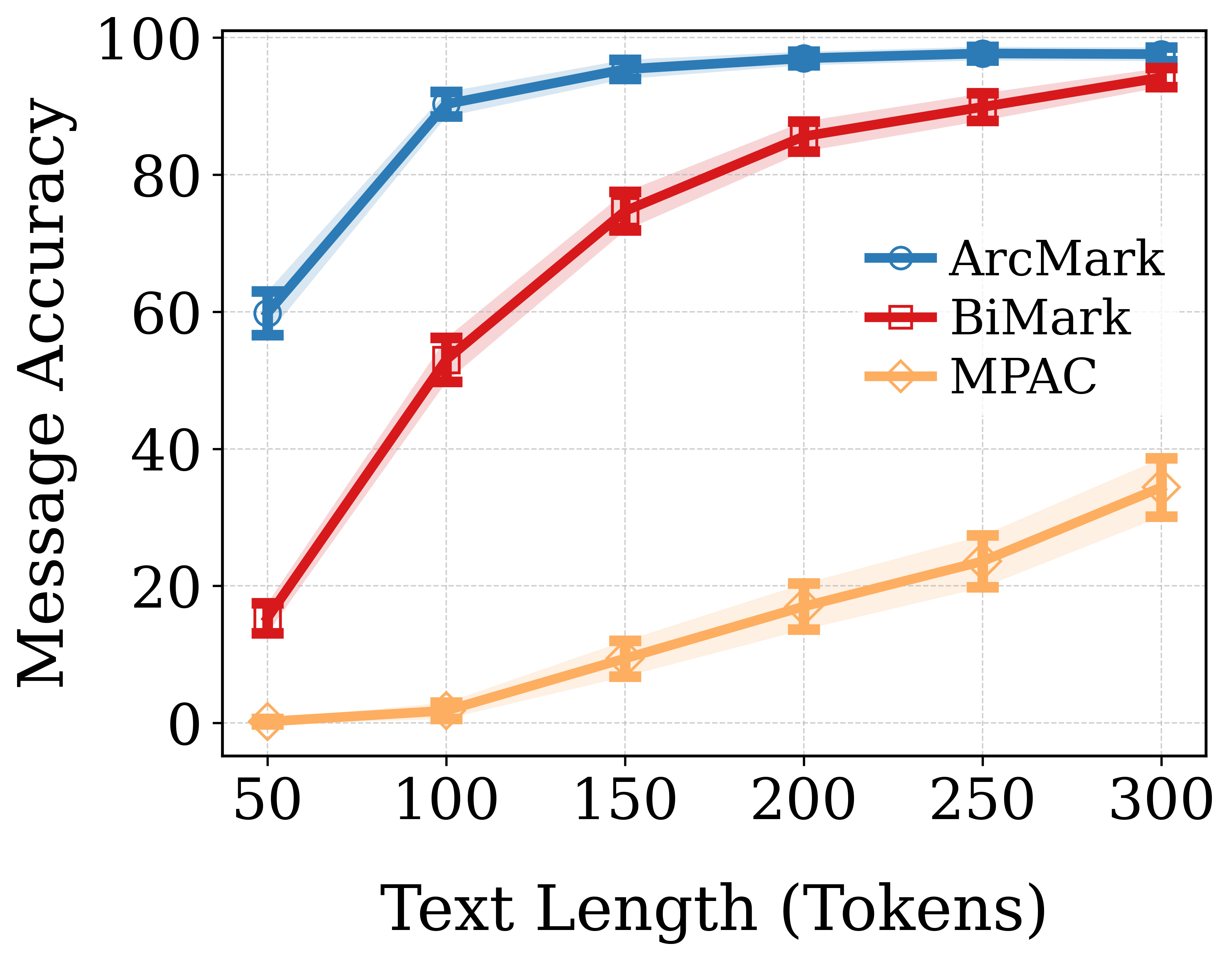}
        \label{fig:llama3-16bit-msg-acc}
    \end{subfigure}
    \hfill
    \begin{subfigure}{0.32\linewidth}
        \centering
        \includegraphics[width=\linewidth]{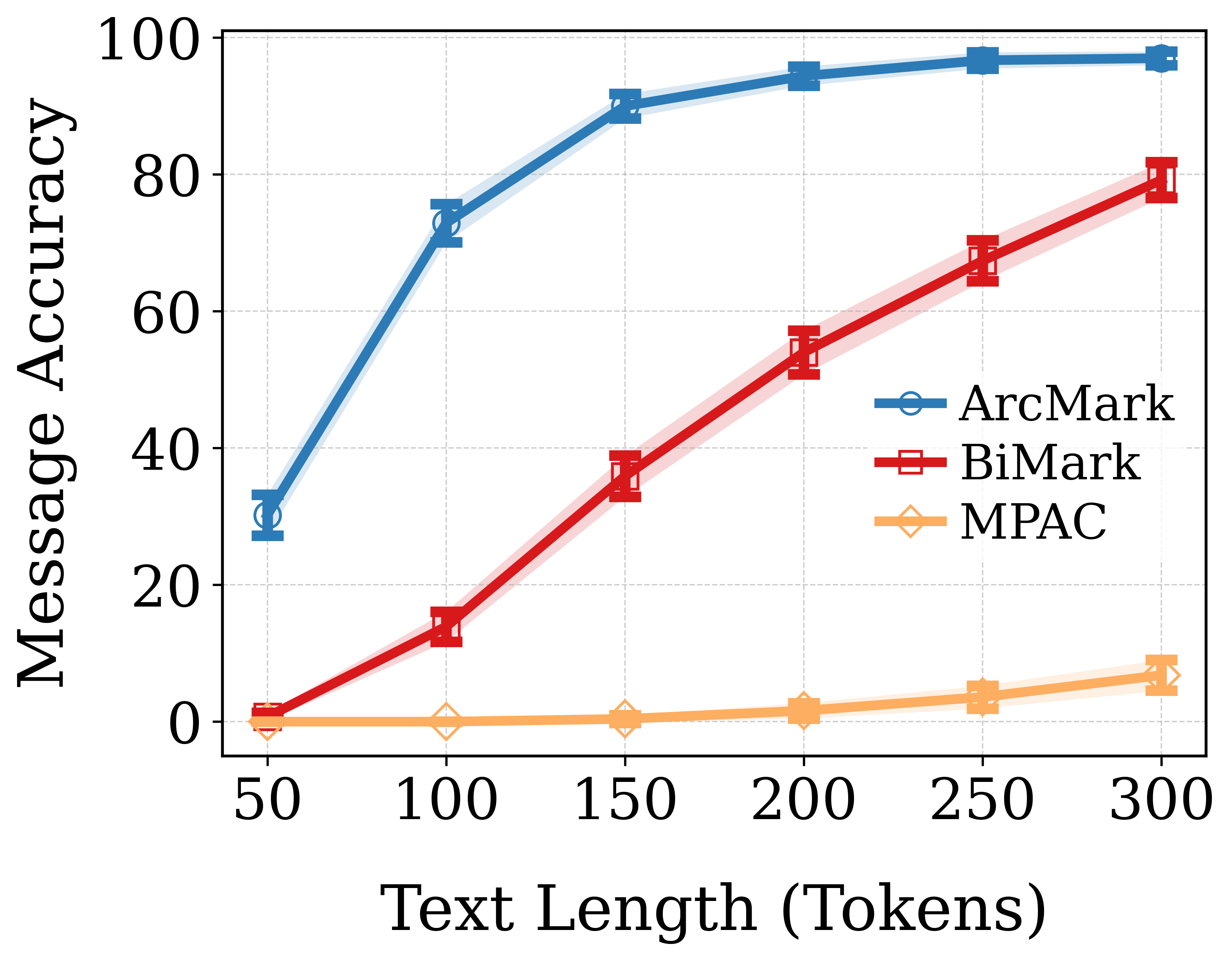}
        \label{fig:llama3-24bit-msg-acc}
    \end{subfigure}
    \hfill
    \begin{subfigure}{0.32\linewidth}
        \centering
        \includegraphics[width=\linewidth]{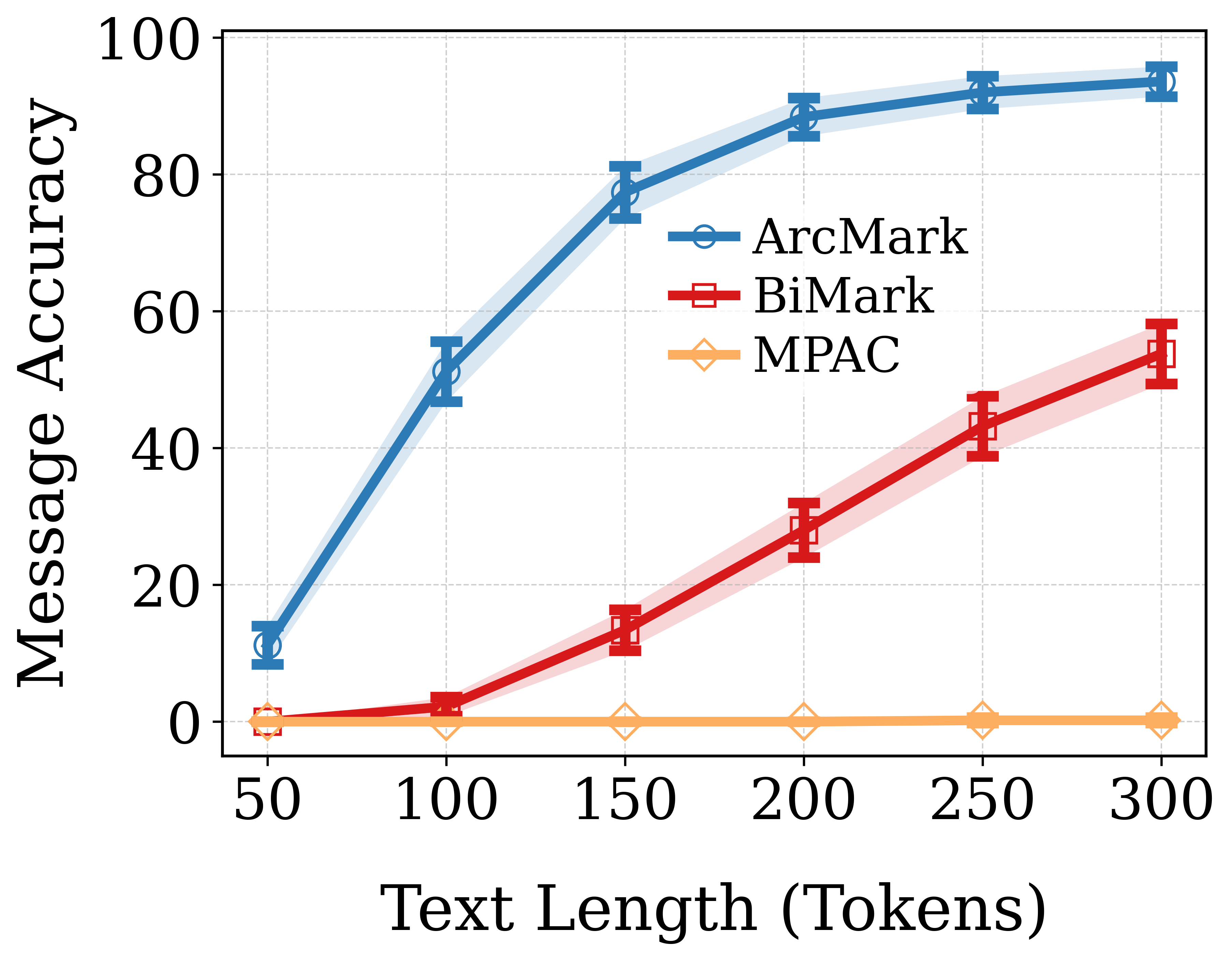}
        \label{fig:llama3-32bit-msg-acc}
    \end{subfigure}

    \vspace{0.5em}

    \includegraphics[width=0.32\linewidth]{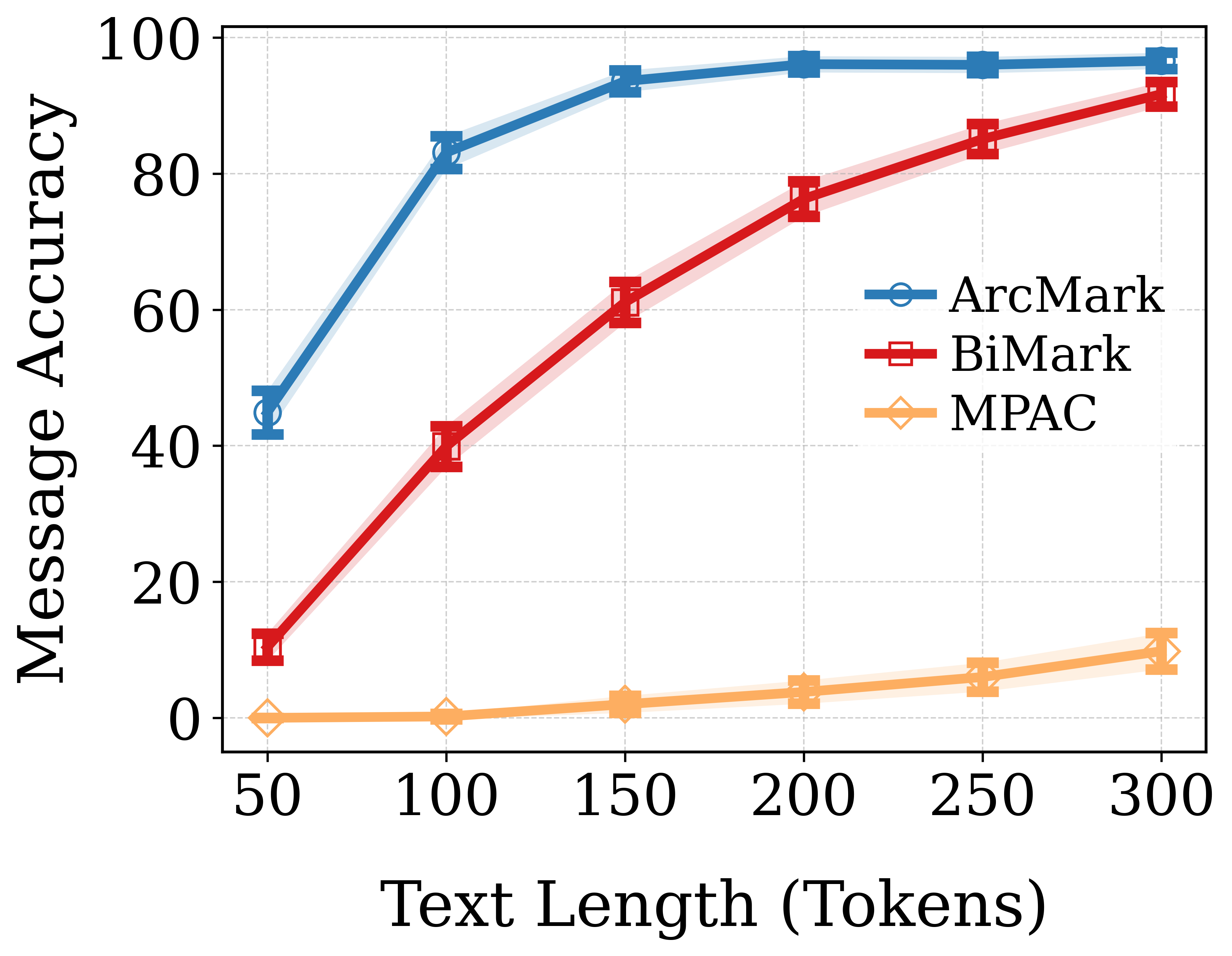}
    \hfill
    \includegraphics[width=0.32\linewidth]{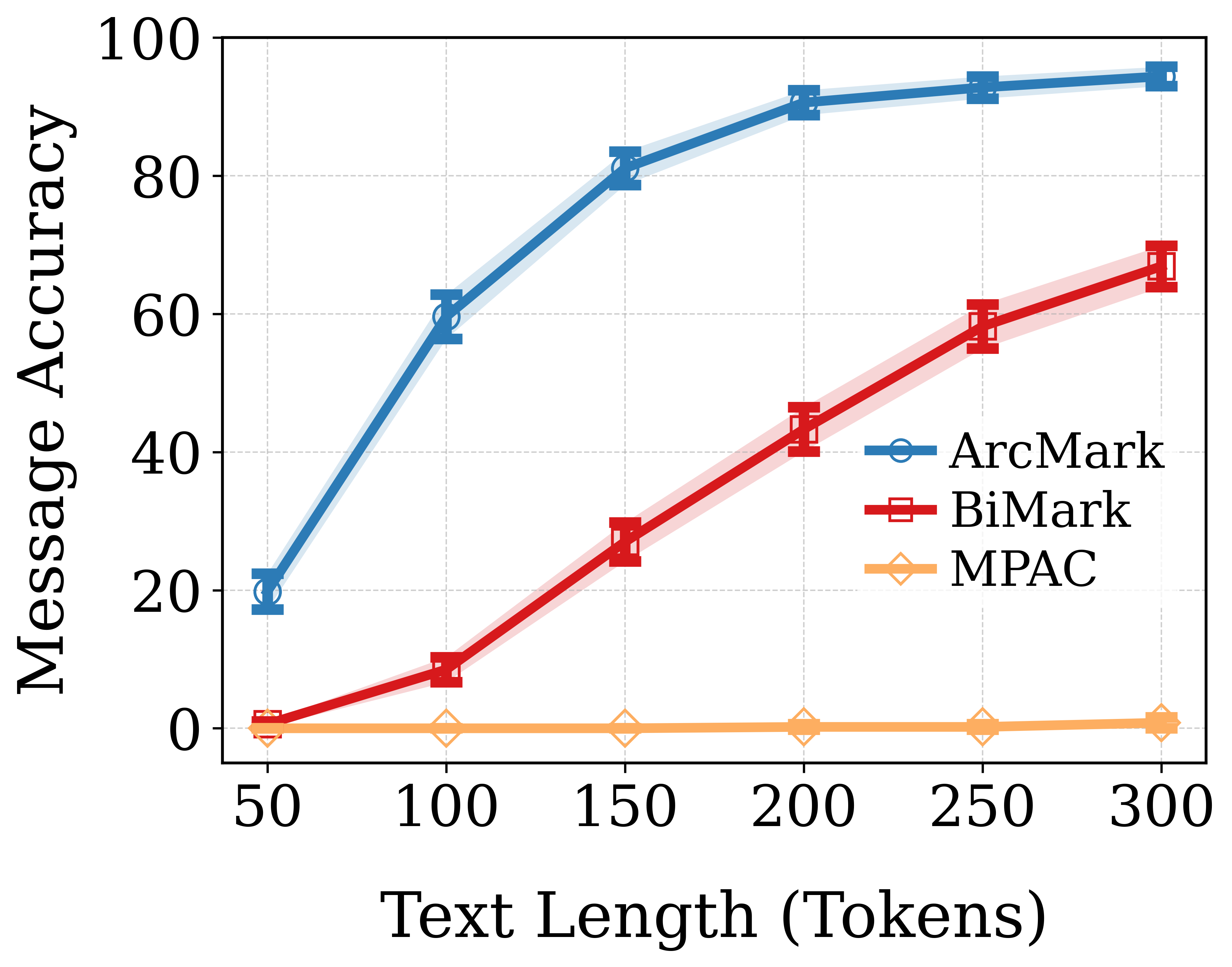}
    \hfill
    \includegraphics[width=0.32\linewidth]{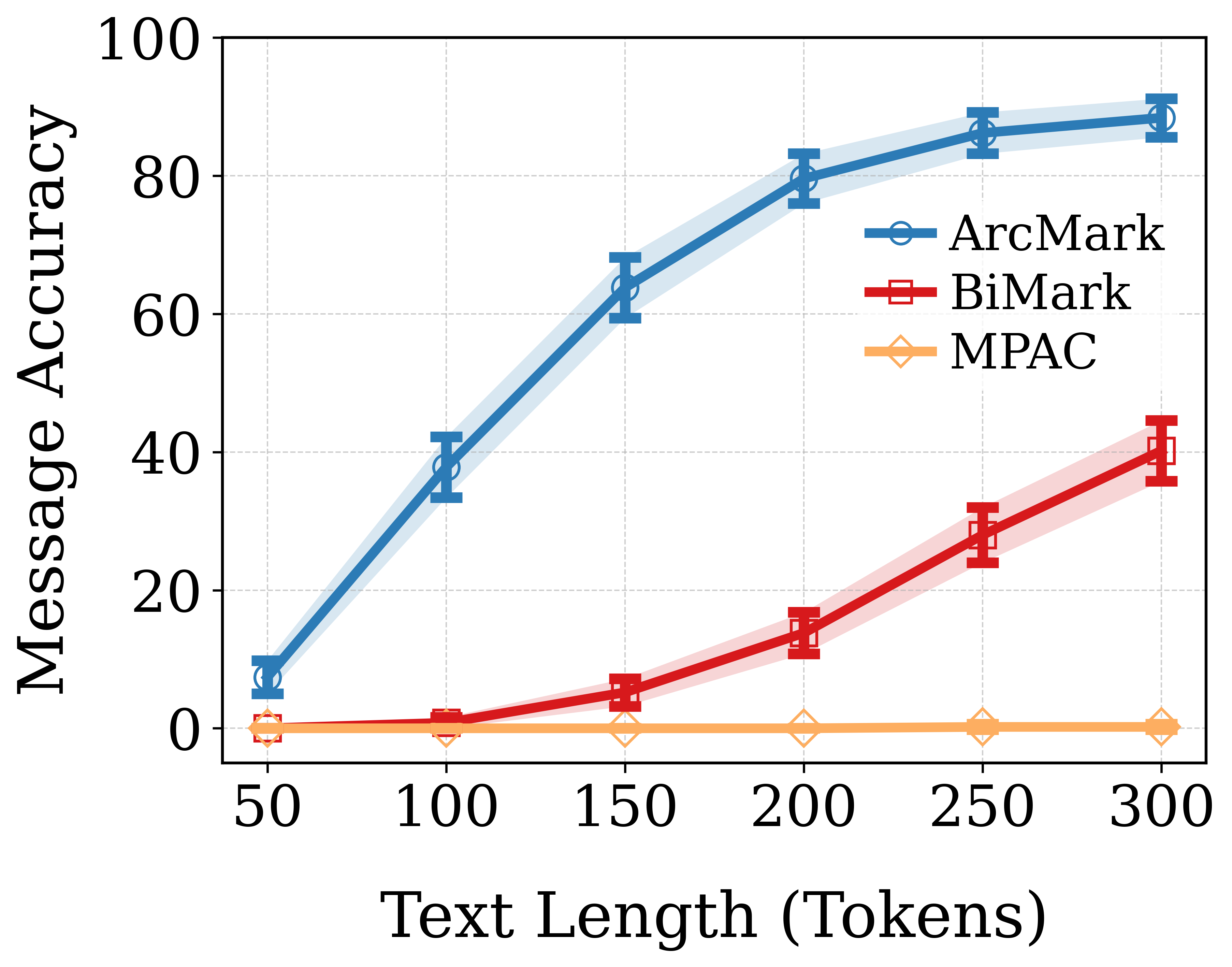}

    \caption{Message accuracy on Llama3-8B (top) and Mistral-7B (bottom) for watermark payloads of 2 bytes (left), 3 bytes (middle), and 4 bytes (right), averaged over 1000 trials for 2- and 3-byte payloads and 500 trials for 4-byte payloads. Error bars indicate SEM. Since recovering the full message requires every bit to be correct, the task becomes increasingly difficult as payload size grows. \arcmark is designed to directly optimize message-level recovery, which explains its growing advantage as payload size increases.}
    \label{fig:llama3-mistral-msg-acc-by-text-length}
\end{figure}
\begin{figure}[h]
\centering
\includegraphics[width=0.32\linewidth]{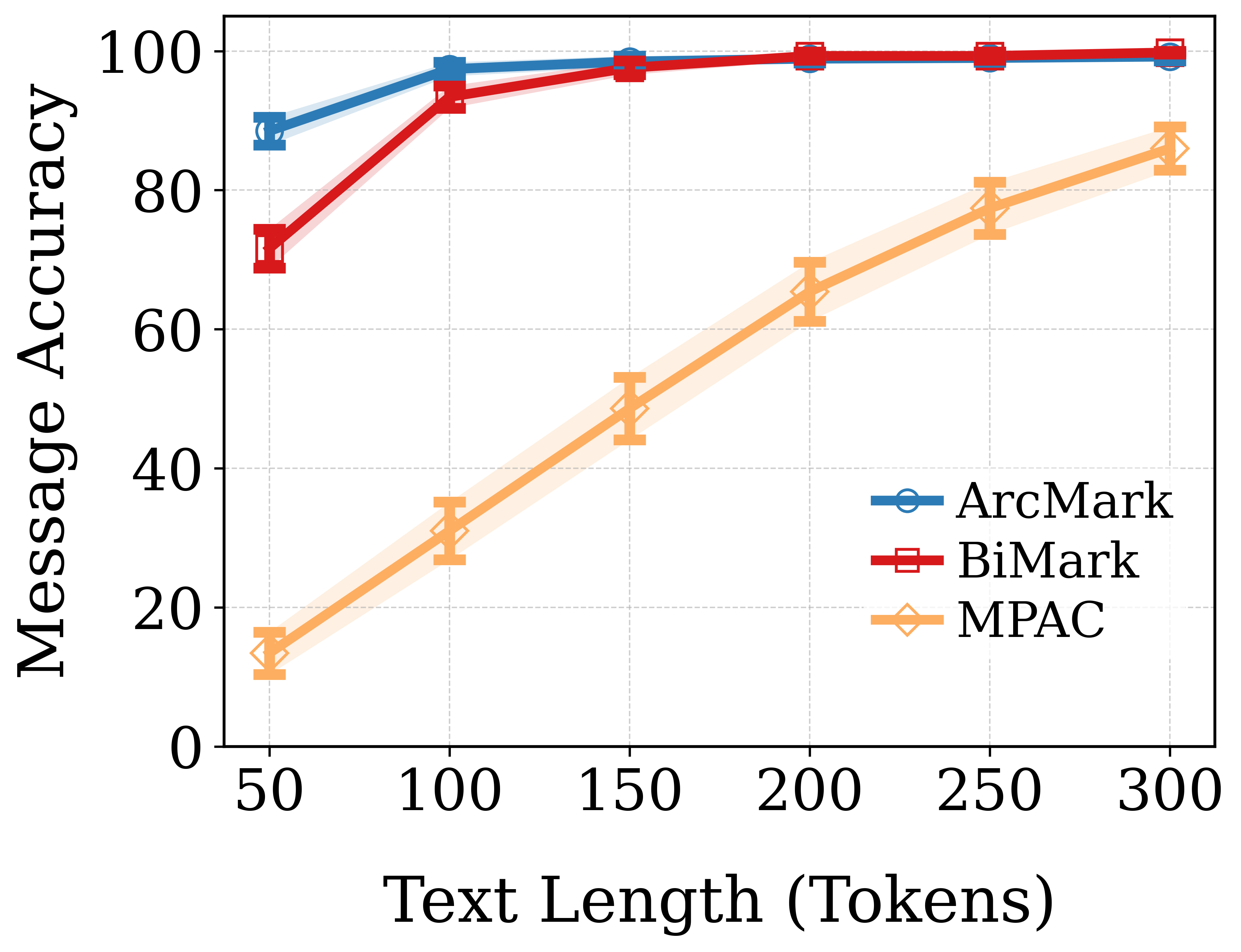}
\hfill
\includegraphics[width=0.32\linewidth]{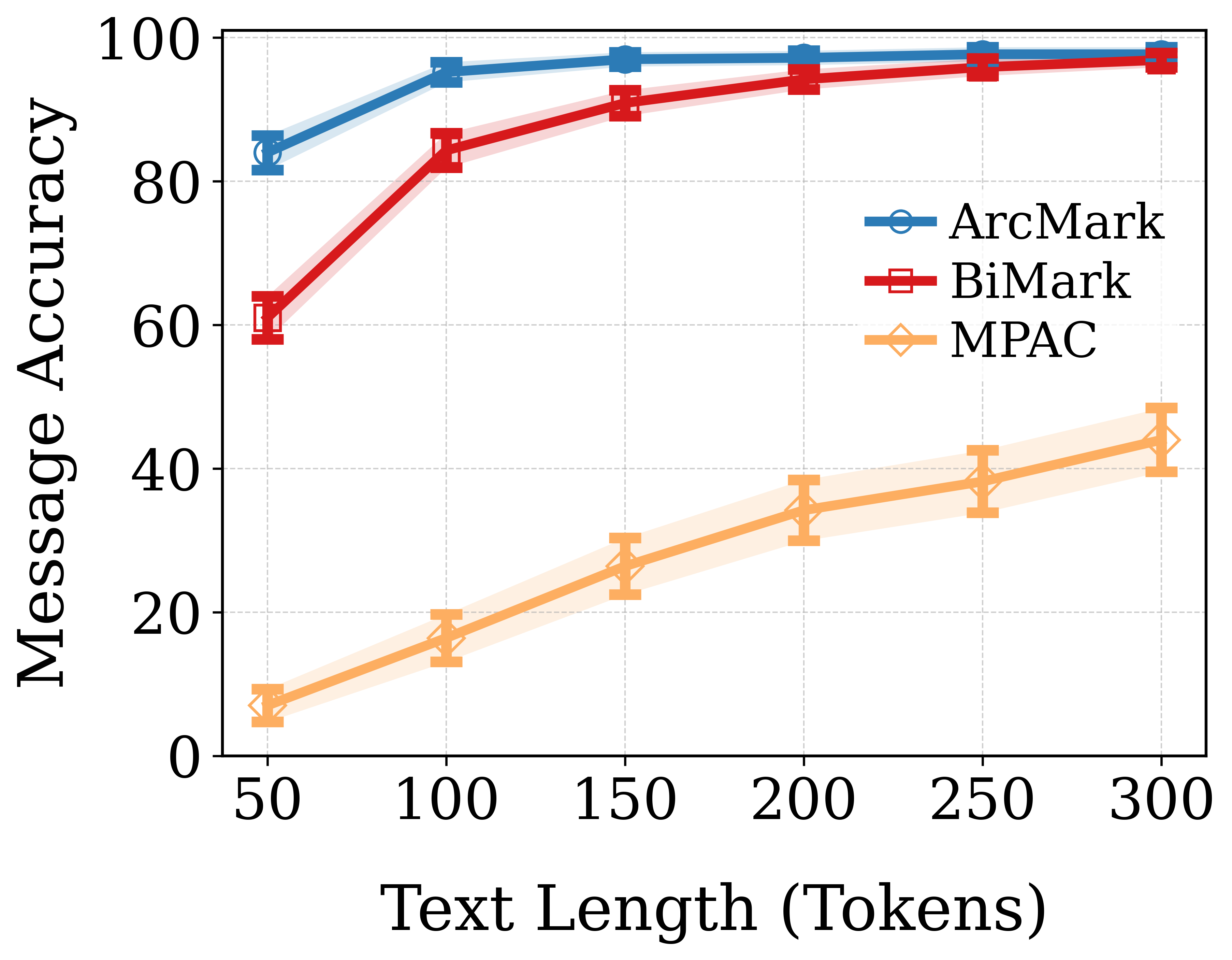}
\hfill
\includegraphics[width=0.32\linewidth]{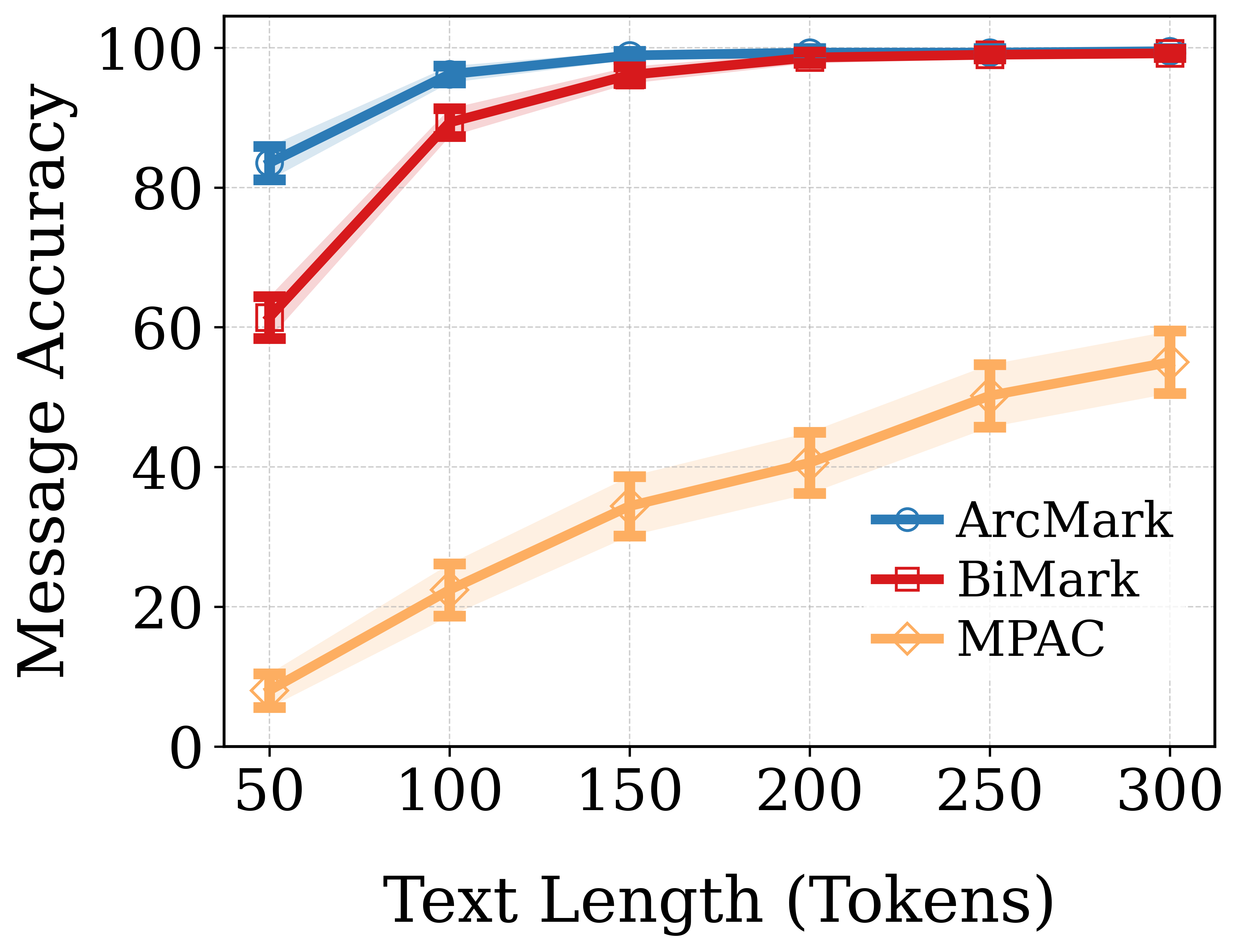}
\caption{Message accuracy on Llama3-8B (left), Qwen3-8B (center), and Mistral-7B (right) for a 1-byte payload, averaged over 1000 trials. Error bars indicate SEM.}
\label{fig:message-acc-8bit}
\end{figure}

%%%%%%%%%%%%%%%%%%%%%%%%%%%%%%%%%%%%%%%%%%%%%%%%
\section{Additional Attack Results}\label{app:attack}

Table~\ref{tab:msg-acc-french} and Table~\ref{tab:msg-acc-french-qwen-16bit} report message accuracy for a 2-byte payload under French back-translation (BT) attacks on Llama3-8B and Qwen3-8B respectively, averaged over 1000 trials. \arcmark consistently outperforms BiMark across all attack ratios, confirming the trends observed in Section~\ref{subsec:attack}.

\begin{table}[h]
\centering
\small
\setlength{\tabcolsep}{4.2pt}
\renewcommand{\arraystretch}{1.08}
\caption{Message accuracy (\%) on Llama3-8B for a 2-byte payload under no attack and French BT attacks, averaged over 1000 trials. $\pm$ indicates SEM. Bold indicates the better method.}
\label{tab:msg-acc-french}
\resizebox{\linewidth}{!}{%
\begin{tabular}{@{}c cc cc cc cc@{}}
\toprule
& \multicolumn{2}{c}{No attack}
& \multicolumn{2}{c}{BT $r=0.1$}
& \multicolumn{2}{c}{BT $r=0.2$}
& \multicolumn{2}{c}{BT $r=0.3$} \\
\cmidrule(lr){2-3}
\cmidrule(lr){4-5}
\cmidrule(lr){6-7}
\cmidrule(l){8-9}
Tokens
& ArcMark & BiMark
& ArcMark & BiMark
& ArcMark & BiMark
& ArcMark & BiMark \\
\midrule
50  & \textbf{59.8} $\pm$ 1.6 & 15.3 $\pm$ 1.1 & \textbf{37.7} $\pm$ 1.5 & 10.1 $\pm$ 1.0 & \textbf{32.8} $\pm$ 1.5 & 8.6 $\pm$ 0.9  & \textbf{25.7} $\pm$ 1.4 & 7.5 $\pm$ 0.8 \\
100 & \textbf{90.3} $\pm$ 0.9 & 53.0 $\pm$ 1.6 & \textbf{70.4} $\pm$ 1.4 & 34.9 $\pm$ 1.5 & \textbf{60.9} $\pm$ 1.5 & 30.4 $\pm$ 1.5 & \textbf{50.7} $\pm$ 1.6 & 25.5 $\pm$ 1.4 \\
150 & \textbf{95.4} $\pm$ 0.7 & 74.7 $\pm$ 1.4 & \textbf{83.9} $\pm$ 1.2 & 56.0 $\pm$ 1.6 & \textbf{76.4} $\pm$ 1.3 & 50.0 $\pm$ 1.6 & \textbf{67.4} $\pm$ 1.5 & 42.5 $\pm$ 1.6 \\
200 & \textbf{97.0} $\pm$ 0.5 & 85.6 $\pm$ 1.1 & \textbf{90.5} $\pm$ 0.9 & 68.8 $\pm$ 1.5 & \textbf{85.4} $\pm$ 1.1 & 63.8 $\pm$ 1.5 & \textbf{79.8} $\pm$ 1.3 & 57.0 $\pm$ 1.6 \\
250 & \textbf{97.7} $\pm$ 0.5 & 89.9 $\pm$ 1.0 & \textbf{94.1} $\pm$ 0.7 & 79.8 $\pm$ 1.3 & \textbf{91.5} $\pm$ 0.9 & 75.5 $\pm$ 1.4 & \textbf{87.7} $\pm$ 1.0 & 68.7 $\pm$ 1.5 \\
300 & \textbf{97.6} $\pm$ 0.5 & 94.2 $\pm$ 0.7 & \textbf{95.5} $\pm$ 0.7 & 85.8 $\pm$ 1.1 & \textbf{93.3} $\pm$ 0.8 & 82.3 $\pm$ 1.2 & \textbf{90.4} $\pm$ 0.9 & 77.1 $\pm$ 1.3 \\
\bottomrule
\end{tabular}%
}
\end{table}

\begin{table}[h]
\centering
\small
\setlength{\tabcolsep}{4.2pt}
\renewcommand{\arraystretch}{1.08}
\caption{Message accuracy (\%) on Qwen3-8B for a 2-byte payload under no attack and French BT attacks, averaged over 1000 trials. $\pm$ indicates SEM. Bold indicates the better method.}
\label{tab:msg-acc-french-qwen-16bit}
\resizebox{\linewidth}{!}{%
\begin{tabular}{@{}c cc cc cc cc@{}}
\toprule
& \multicolumn{2}{c}{No attack}
& \multicolumn{2}{c}{BT $r=0.1$}
& \multicolumn{2}{c}{BT $r=0.2$}
& \multicolumn{2}{c}{BT $r=0.3$} \\
\cmidrule(lr){2-3}
\cmidrule(lr){4-5}
\cmidrule(lr){6-7}
\cmidrule(l){8-9}
Tokens
& ArcMark & BiMark
& ArcMark & BiMark
& ArcMark & BiMark
& ArcMark & BiMark \\
\midrule
50  & \textbf{46.6} $\pm$ 1.6 & 9.0 $\pm$ 0.9  & \textbf{30.7} $\pm$ 1.5 & 6.4 $\pm$ 0.8  & \textbf{26.1} $\pm$ 1.4 & 5.4 $\pm$ 0.7  & \textbf{22.1} $\pm$ 1.3 & 4.7 $\pm$ 0.7 \\
100 & \textbf{80.5} $\pm$ 1.3 & 39.5 $\pm$ 1.5 & \textbf{59.7} $\pm$ 1.6 & 28.6 $\pm$ 1.4 & \textbf{53.2} $\pm$ 1.6 & 23.6 $\pm$ 1.3 & \textbf{44.7} $\pm$ 1.6 & 19.9 $\pm$ 1.3 \\
150 & \textbf{89.8} $\pm$ 1.0 & 59.5 $\pm$ 1.6 & \textbf{77.2} $\pm$ 1.3 & 46.5 $\pm$ 1.6 & \textbf{70.8} $\pm$ 1.4 & 41.6 $\pm$ 1.6 & \textbf{63.0} $\pm$ 1.5 & 37.0 $\pm$ 1.5 \\
200 & \textbf{91.6} $\pm$ 0.9 & 73.7 $\pm$ 1.4 & \textbf{84.0} $\pm$ 1.2 & 60.7 $\pm$ 1.5 & \textbf{80.0} $\pm$ 1.3 & 56.0 $\pm$ 1.6 & \textbf{73.5} $\pm$ 1.4 & 49.2 $\pm$ 1.6 \\
250 & \textbf{92.2} $\pm$ 0.8 & 79.1 $\pm$ 1.3 & \textbf{88.3} $\pm$ 1.0 & 68.0 $\pm$ 1.5 & \textbf{85.3} $\pm$ 1.1 & 61.4 $\pm$ 1.5 & \textbf{80.9} $\pm$ 1.2 & 56.0 $\pm$ 1.6 \\
300 & \textbf{93.7} $\pm$ 0.8 & 83.1 $\pm$ 1.2 & \textbf{88.6} $\pm$ 1.0 & 72.4 $\pm$ 1.4 & \textbf{86.4} $\pm$ 1.1 & 67.1 $\pm$ 1.5 & \textbf{83.4} $\pm$ 1.2 & 61.8 $\pm$ 1.5 \\
\bottomrule
\end{tabular}%
}
\end{table}

%%%%%%%%%%
\section{Perplexity}\label{app:ppl}
We use perplexity as a proxy for generation quality and to assess whether \arcmark preserves the model’s output distribution. Perplexity is computed on text generated with and without watermarking, evaluated under the same language model used for generation. This controls for model mismatch, so that observed differences can be attributed directly to the watermarking method. Table~\ref{tab:qwen-ppl-no-attack} in Section~\ref{subsec:ppl} and Table~\ref{tab:llama-ppl-no-attack} below show that \arcmark consistently achieves perplexity closer to the non-watermarked baseline than BiMark across all payload sizes and token lengths for Qwen3-8B and Llama3-8B, respectively.

\begin{table}[h]
\centering
\small
\setlength{\tabcolsep}{3.2pt}
\renewcommand{\arraystretch}{1.08}
\caption{Perplexity on Llama3-8B for 2-, 3-, and 4-byte watermarks, averaged over 1000 trials for 2 and 3 bytes and 500 trials for 4 bytes. Bold indicates the method closer to the non-watermarked baseline.}
\label{tab:llama-ppl-no-attack}
\resizebox{\linewidth}{!}{%
\begin{tabular}{@{}c c cc cc cc@{}}
\toprule
& & \multicolumn{2}{c}{2-byte} & \multicolumn{2}{c}{3-byte} & \multicolumn{2}{c}{4-byte} \\
\cmidrule(lr){3-4}
\cmidrule(lr){5-6}
\cmidrule(l){7-8}
Tokens & No watermark & ArcMark & BiMark & ArcMark & BiMark & ArcMark & BiMark \\
\midrule
50  & 6.180 $\pm$ 0.089 & \textbf{6.178 $\pm$ 0.095} & 6.427 $\pm$ 0.097 & \textbf{6.045 $\pm$ 0.084} & 6.645 $\pm$ 0.109 & \textbf{6.172 $\pm$ 0.128} & 6.463 $\pm$ 0.132 \\
100 & 5.772 $\pm$ 0.073 & \textbf{5.678 $\pm$ 0.071} & 5.975 $\pm$ 0.075 & \textbf{5.668 $\pm$ 0.070} & 6.084 $\pm$ 0.084 & \textbf{5.768 $\pm$ 0.106} & 6.078 $\pm$ 0.105 \\
150 & 5.604 $\pm$ 0.067 & \textbf{5.486 $\pm$ 0.066} & 5.840 $\pm$ 0.070 & \textbf{5.482 $\pm$ 0.065} & 5.929 $\pm$ 0.073 & \textbf{5.473 $\pm$ 0.096} & 5.925 $\pm$ 0.097 \\
200 & 5.507 $\pm$ 0.063 & \textbf{5.357 $\pm$ 0.062} & 5.782 $\pm$ 0.067 & \textbf{5.313 $\pm$ 0.062} & 5.882 $\pm$ 0.071 & \textbf{5.305 $\pm$ 0.089} & 5.843 $\pm$ 0.093 \\
250 & 5.436 $\pm$ 0.061 & \textbf{5.260 $\pm$ 0.061} & 5.766 $\pm$ 0.065 & \textbf{5.197 $\pm$ 0.060} & 5.845 $\pm$ 0.067 & \textbf{5.203 $\pm$ 0.087} & 5.815 $\pm$ 0.092 \\
300 & 5.373 $\pm$ 0.059 & \textbf{5.157 $\pm$ 0.060} & 5.768 $\pm$ 0.065 & \textbf{5.107 $\pm$ 0.059} & 5.831 $\pm$ 0.066 & \textbf{5.094 $\pm$ 0.085} & 5.815 $\pm$ 0.090 \\
\bottomrule
\end{tabular}%
}
\end{table}
%%%%%%%%%%%%%%%
%%%%%%%%%%%%%%%%%%%%%%%%%%%%%%%%%%%%%%%%%%%%

\section{Quality on Downstream Tasks}\label{app:downstream}

We evaluate ArcMark on three downstream tasks designed to test different aspects of output quality. Extractive question answering (Q\&A) tests whether watermarking preserves the factual content necessary to answer a question. Long-context code completion (LCC) \cite{lcc} tests the fidelity of a deterministic ground truth in an open-ended code-generation setting. HumanEval \cite{humaneval} pass@k tests the functional correctness of generated code with watermarks. In each setting, we compare ArcMark against both an unwatermarked baseline and BiMark on Llama3-8B and Qwen3-8B with 2-byte messages embedded in the generated text.

\paragraph{Q\&A:} Following the SQuAD v2 setup \cite{qanda}, we sample 200 answerable examples from the validation set with contexts of at least 200 characters. For each example, we use the first 100 tokens of the \emph{gold passage} as the prompt and generate a 100-token continuation under three conditions: (i) no watermark, (ii) ArcMark, and (iii) BiMark. We then run a pretrained extractive QA model, (distilbert-base-cased-distilled-squad \cite{distilbert}) on the resulting prompt–continuation pair together with the \emph{gold question}. Predictions are evaluated against the human reference answers using the standard SQuAD F1 and Exact Match (EM) metrics (see Table~\ref{tab:qa_squad}). As an upper bound, we also evaluate the QA model directly on the original gold passage. The gap between the original-passage condition and the unwatermarked continuation in Table~\ref{tab:qa_squad} reflects the loss incurred by replacing the \emph{gold passage} with LLM generated text. Relative to this generation baseline, both ArcMark and BiMark remain within one standard error of the unwatermarked condition on both F1 and EM. This suggests that, at 2 bytes, watermarking causes little to no additional degradation in downstream Q\&A performance.

\begin{table}[t]
\centering
\small
\caption{Q\&A results on SQuAD v2 with 2-byte watermarking.}
\label{tab:qa_squad}
\begin{tabular}{lcccc}
\toprule
& \multicolumn{2}{c}{\textbf{Llama3-8B}} & \multicolumn{2}{c}{\textbf{Qwen3-8B}} \\
\cmidrule(lr){2-3} \cmidrule(lr){4-5}
\textbf{Condition} & \textbf{F1} & \textbf{EM} & \textbf{F1} & \textbf{EM} \\
\midrule
Original passage (upper bound) & 0.768 $\pm$ 0.027 & 0.695 $\pm$ 0.033 & 0.768$\pm$0.027 & 0.695$\pm$0.033\\
Unwatermarked continuation     & 0.646 $\pm$ 0.031 & 0.56 $\pm$ 0.035 & 0.609$\pm$0.031 & 0.51$\pm$0.035\\
ArcMark                        & 0.615 $\pm$ 0.032 & 0.535 $\pm$ 0.035 & 0.635$\pm$0.031 & 0.56$\pm$0.035\\
BiMark                         & 0.617 $\pm$ 0.032 & 0.54 $\pm$ 0.035 & 0.599$\pm$0.032& 0.52$\pm$0.035\\
\bottomrule
\end{tabular}
\end{table}

\paragraph{LCC:} We sample 200 Python examples from Microsoft’s LCC dataset \cite{lcc}. For each example, we use the last 256 context tokens as the prompt and generate a 64-token continuation under: (i) no watermark, (ii) ArcMark, and (iii) BiMark. We compare each generated continuation against the reference completion using two standard metrics: character-level Exact Match (EM) and edit similarity, defined as $1-\frac{\text{Levenshtein distance}}{\text{max length}}$. As shown in Table~\ref{tab:lcc_python}, the performance differences across the three conditions are small relative to the standard errors. ArcMark remains closer to the unwatermarked baseline than BiMark on both Exact Match and edit similarity, and neither watermarking method causes a statistically meaningful drop in Python LCC performance at 2 bytes.
\begin{table}[h]
\centering
\small
\caption{LCC results on Python examples from LCC with 2-byte watermarking.}
\label{tab:lcc_python}
\begin{tabular}{lcccc}
\toprule
& \multicolumn{2}{c}{\textbf{Llama3-8B}} & \multicolumn{2}{c}{\textbf{Qwen3-8B}} \\
\cmidrule(lr){2-3} \cmidrule(lr){4-5}
\textbf{Condition} & \textbf{EM} & \textbf{Edit Sim.} & \textbf{EM} & \textbf{Edit Sim.} \\
\midrule
Unwatermarked & 0.205 $\pm$ 0.029 & 0.562 $\pm$ 0.022 & 0.245$\pm$ 0.031& 0.594$\pm$0.023\\
ArcMark       & 0.190 $\pm$ 0.028 & 0.540 $\pm$ 0.022 & 0.225$\pm$0.03& 0.585$\pm$0.0223\\
BiMark        & 0.180 $\pm$ 0.027 & 0.526 $\pm$ 0.022 & 0.25$\pm$0.031 & 0.598$\pm$0.0227\\
\bottomrule
\end{tabular}
\end{table}

\paragraph{HumanEval pass@k:} We evaluate multi-bit watermarking under functional correctness on the OpenAI HumanEval benchmark \cite{humaneval}. For each of the 164 coding problems, we use the function signature and docstring as the prompt and generate a 256-token completion under: (i) no watermark, (ii) ArcMark, and (iii) BiMark. Each completion is appended to the prompt to form a candidate program, which is then executed against the problem’s hidden unit tests. We report pass@5, based on the pass@k metric introduced in \cite{humaneval}. The results are shown in Table~\ref{tab:humaneval}, which shows that all three methods (unwatermarked, \arcmark, and Bimark) have comparable performance.
\begin{table}[t]
\centering
\small
\caption{Code generation results on HumanEval with 2-byte watermarking. We report pass@5 $\pm$ SEM.}
\label{tab:humaneval}

\begin{tabular}{lcc}
\toprule
\textbf{Condition} & \textbf{Llama3-8B} & \textbf{Qwen3-8B} \\
\midrule
Unwatermarked  & 0.748 $\pm$ 0.030 & 0.785 $\pm$ 0.027 \\
ArcMark     & 0.700 $\pm$ 0.030 & 0.797 $\pm$ 0.028 \\
BiMark      & 0.651 $\pm$ 0.033 & 0.769 $\pm$ 0.029 \\
\bottomrule
\end{tabular}

\end{table}
\vspace{-3mm}
\section{Zero-Bit Detection Results}\label{app:falsealarm}
\vspace{-3mm}
We evaluate zero-bit detection on Llama3-8B, Qwen3-8B, and Mistral-7B. The negative class consists of 1000 human-written articles from the C4 validation set, never passed through the model, and the positive class consists of 1000 watermarked texts generated by the model using C4 articles as prompts, with a 1-byte watermark. We report true positive rate (TPR) at a fixed false positive rate (FPR) of 1\% (at most 1 in 100 human-written texts incorrectly flagged as watermarked), with bootstrap SEM over 1000 resamples. Results in Table~\ref{tab:tpr-fpr1-all-models} show BiMark achieves modestly higher TPR at shorter text lengths, but both \arcmark and BiMark reach near-perfect detection accuracy from 150 tokens onward.
%The detection threshold is swept over all unique score values from both classes to trace the full ROC curve, and TPR at the target FPR is extracted from this curve. 
\begin{table}[h]
\centering
\scriptsize
\setlength{\tabcolsep}{4.2pt}
\renewcommand{\arraystretch}{1.08}
\caption{TPR (\%) at FPR = 1\% for zero-bit detection with a 1-byte watermark, across three models.}
\label{tab:tpr-fpr1-all-models}
\resizebox{\linewidth}{!}{%
\begin{tabular}{@{}c cc cc cc@{}}
\toprule
& \multicolumn{2}{c}{Llama3-8B}
& \multicolumn{2}{c}{Mistral-7B}
& \multicolumn{2}{c}{Qwen3-8B} \\
\cmidrule(lr){2-3}
\cmidrule(lr){4-5}
\cmidrule(l){6-7}
Tokens
& ArcMark & BiMark
& ArcMark & BiMark
& ArcMark & BiMark \\
\midrule
50
& 86.70 $\pm$ 0.05 & 85.50 $\pm$ 0.08
& 75.70 $\pm$ 0.09 & 84.60 $\pm$ 0.10
& 72.80 $\pm$ 0.07 & 80.60 $\pm$ 0.07 \\

100
& 97.50 $\pm$ 0.02 & 98.80 $\pm$ 0.01
& 94.70 $\pm$ 0.04 & 98.90 $\pm$ 0.01
& 91.70 $\pm$ 0.04 & 95.90 $\pm$ 0.02 \\

150
& 99.20 $\pm$ 0.01 & 99.80 $\pm$ 0.00
& 98.80 $\pm$ 0.01 & 99.60 $\pm$ 0.01
& 95.90 $\pm$ 0.02 & 98.30 $\pm$ 0.02 \\

200
& 100.00 $\pm$ 0.00 & 100.00 $\pm$ 0.00
& 99.20 $\pm$ 0.01 & 99.80 $\pm$ 0.00
& 97.50 $\pm$ 0.02 & 97.10 $\pm$ 0.02 \\

250
& 99.90 $\pm$ 0.00 & 100.00 $\pm$ 0.00
& 99.60 $\pm$ 0.01 & 100.00 $\pm$ 0.00
& 97.90 $\pm$ 0.02 & 98.70 $\pm$ 0.01 \\
\bottomrule
\end{tabular}%
}
\end{table}
\vspace{-3mm}
\section{Ablation Study on Side Information Resolution}
\label{app:ablation}
\vspace{-3mm}
Recall that $r$ denotes the number of discrete side information values used in the optimal transport formulation (see Eq.~\ref{eq:cost_matrix}). This parameter determines the size of the OT cost matrix and therefore the computational cost of the Sinkhorn solver. We study how varying $r \in \{64, 256, 512\}$ affects message accuracy. As shown in Figure~\ref{fig:side_info_resolution_ablation}, performance is similar across all three settings, indicating that increasing the side information resolution has little practical impact on extraction accuracy. Since larger values of $r$ incur higher computational overhead, these results indicate that a small value such as $r=64$ is sufficient in practice.

\begin{figure*}[h]
    \centering
    \begin{subfigure}[t]{0.45\linewidth}
        \centering
        \includegraphics[width=\linewidth]{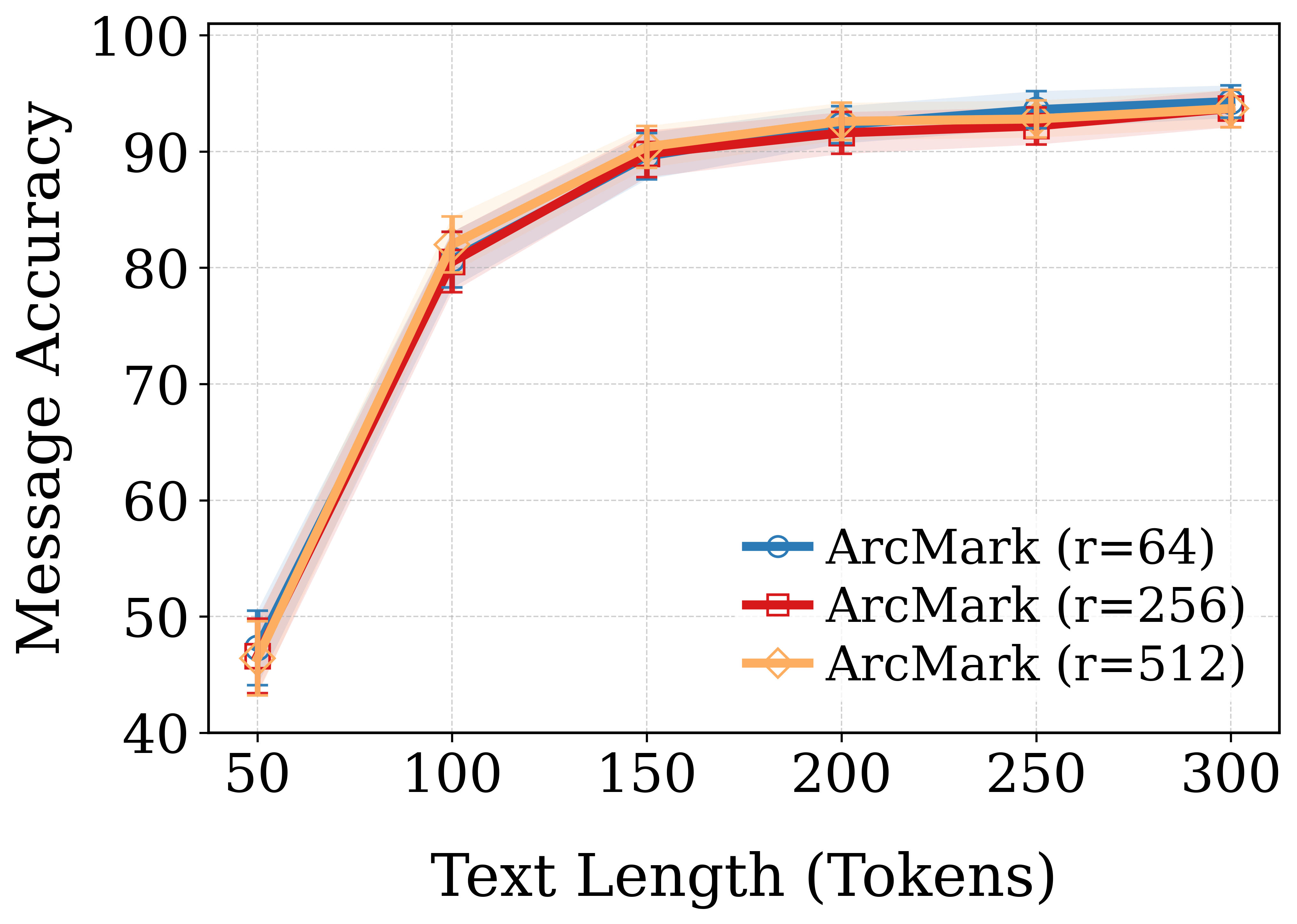}
    \end{subfigure}
    \hfill
    \begin{subfigure}[t]{0.45\linewidth}
        \centering
        \includegraphics[width=\linewidth]{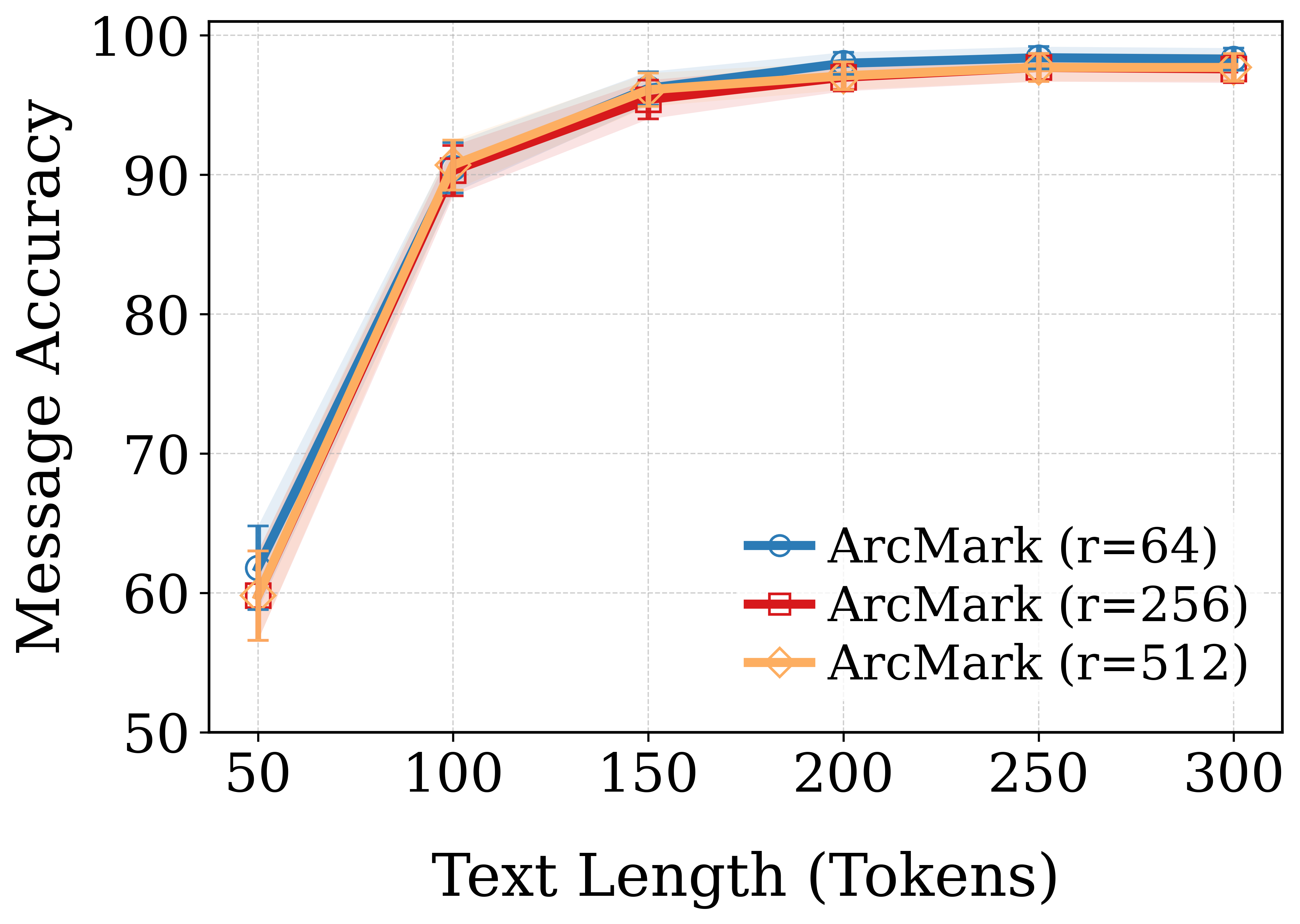}
    \end{subfigure}
    \caption{Message accuracy for Qwen3-8B (left) and Llama3-8B (right) when varying the number of discrete side information values ($r = 64, 256, 512$) for 2-byte watermarks, averaged over 1000 trials. Error bars indicate SEM. Performance is comparable across all values of $r$, indicating that small side information resolutions such as $r=64$ are sufficient in practice and lead to a more efficient optimal transport solver.}
    \label{fig:side_info_resolution_ablation}
\end{figure*}

\end{document}